\documentclass[twoside,11pt]{article}
\usepackage{jair, rawfonts}

\usepackage[utf8]{inputenc}
\usepackage{graphicx}
\usepackage{amsmath}
\usepackage{amssymb}
\usepackage{csquotes}
\usepackage{centernot}
\usepackage{cleveref}

\usepackage{todonotes}
\usepackage{tikz}
\usepackage{subfig}
\usepackage{amsthm} 

\newtheoremstyle{break}
  {\topsep}{\topsep}%
  {\itshape}{}%
  {\bfseries}{}%
  {\newline}{}%
\theoremstyle{break}

\newtheorem{example}{Example}
\newtheorem{proposition}{Proposition}

\newtheorem{definition}{Definition}
\newtheorem{corollary}{Corollary}

\newcommand{\vsim}{\mathrel{\scalebox{1}[1.5]{$\shortmid$}\mkern-3.1mu\raisebox{0.15ex}{$\sim$}}}
\ShortHeadings{Abstract Argumentation and the Rational Man}
{To appear in the Journal of Logic and Computation}
\begin{document}

\title{Abstract Argumentation and the Rational Man}
\author{\name Timotheus Kampik \email tkampik@cs.umu.se \\
       \name Juan Carlos Nieves \email jcnieves@cs.umu.se \\
       \addr Department of Computing Science, Umeå University \\ 90187 Umeå, Sweden}

\maketitle

    \begin{abstract}
        Abstract argumentation has emerged as a method for non-monotonic reasoning that has gained popularity in the symbolic artificial intelligence community.
        In the literature, the different approaches to abstract argumentation that were refined over the years are typically evaluated from a formal logics perspective; an analysis that is based on models of \emph{economically rational} decision-making does not exist.
        In this paper, we work towards addressing this issue by analyzing abstract argumentation from the perspective of the \emph{rational man} paradigm in microeconomic theory. To assess under which conditions abstract argumentation-based decision-making can be considered \emph{economically rational}, we derive \emph{reference independence} as a non-monotonic inference property from a formal model of economic rationality and create a new argumentation principle that ensures compliance with this property. We then compare the reference independence principle with other reasoning principles, in particular with cautious monotony and rational monotony. 
        We show that the argumentation semantics as proposed in Dung's seminal paper, as well as other semantics we evaluate -- with the exception of naive semantics and the SCC-recursive CF2 semantics -- violate the reference independence principle.
        Consequently, we investigate how structural properties of argumentation frameworks impact the reference independence principle, and identify \emph{cyclic expansions} (both even and odd cycles) as the root of the problem.
        Finally, we put reference independence into the context of preference-based argumentation and show that for this argumentation variant, which explicitly models preferences, reference independence cannot be ensured in a straight-forward manner.
    \end{abstract}

%
%
%

\section{Introduction}
In recent decades, a large body of research emerged on argumentation as a method to instill intelligence, \emph{i.e.}, reasoning capabilities, into computing systems.
In this context, a popular theory for formal argumentation is so-called \emph{abstract argumentation}, as initially developed by Dung~\cite{dung1995acceptability}, which is frequently used as the foundation of research on formal argumentation and hence spawned a variety of frameworks that extend the initial work~\cite{baroni2018abstract}.
A key aspect of abstract argumentation research is the definition and evaluation of different \emph{argumentation semantics}.
An argumentation semantics is a decision function that specifies how an argumentation framework (a graph of arguments) should be resolved, \emph{i.e.}, which arguments can be considered valid conclusions and which cannot.
A variety of argumentation semantics, as well as a range of advanced frameworks that aim to augment Dung's basic definition of an argumentation framework exist; \emph{i.e.}, the way how an argumentation framework is to be resolved depends on the application scenario and the intended meaning of arguments and attacks.

Determining the arguments of an argumentation framework that can be considered valid conclusions is a decision process.
From the perspective of economic theory, a classical model of decision-making is the notion of the \emph{rational man} (see, \emph{e.g.}, Rubinstein~\cite{rubinstein1998modeling}), whose decisions imply clear and consistent preferences.
While the rational man as a sufficiently precise formal model of human decision-making has been debunked by a body of empirical research~\cite{kahneman2003maps}, it can still be considered relevant as a cornerstone of economic theory, for example for \emph{prescriptive} modeling of an ideal agent that strictly optimizes according to clear preferences when determining the decision outcome.
Aligning abstract argumentation with the rational man paradigm requires, as this paper shows, the introduction of a new \emph{argumentation principle} -- an axiomatic requirement to an argumentation semantics~\cite{van2017principle} -- that existing argumentation semantics typically do not satisfy.
In this paper, we use a novel\footnote{While Caminada introduces \emph{rationality postulates} for abstract argumentation semantics~\cite{caminada2017rationality}, his work focuses on rationality from a logics perspective, \emph{i.e.}, on properties of argumentation semantics with regards to specific argumentation frameworks, without considering the \emph{expansions} of any argumentation framework. Also, let us note that an application of the \emph{rational man's argumentation principles} we establish in this work to Amgoud's and Cayrol's preference-based argumentation is provided in Section~\ref{related}.}
formal approach to explore the intersection of abstract argumentation  as a non-monotonic reasoning method and (bounded) economic rationality and show that abstract argumentation can lead to economically not rational decision outcomes.

Let us provide an example that illustrates the problem this paper focuses on.
An agent ${\cal A}_1$ uses abstract argumentation for deciding on a subset $A^{*}$ of a set $A = \{a, b\}$ (\emph{e.g.,} to determine which actions in a set $Acts \subseteq A$ should be executed and/or which epistemic propositions $Ep \subseteq A$ should be considered true).
For this, the agent constructs the \emph{attack} $(a, b)$, \emph{i.e.}, $a$ attacks $b$.
Given the attack relation $\{(a, b)\}$, the arguments $\{a, b\}$ are typically resolved as $\{a\}$. Let us refer to the constructed argumentation framework $(AR, Attacks)= (\{a, b\}, \{(a, b)\})$ as $AF$.
According to the rational economic man paradigm, the agent's decision of $\{a\}$ implies that the agent prefers accepting $\{a\}$ over all other possible options; \emph{i.e.}, the agent constructs the following (implicit) preference relation: $\{a\} \succeq \{a, b\},  \{a\} \succeq \{b\}, \{a\} \succeq \{ \}$.
Now, let us assume that a second agent ${\cal A}_2$ consults ${\cal A}_1$ to consider a third item $c$, \emph{i.e.}, the set of which the agent determines a subset of valid conclusions is now $\{a, b, c\}$.
${\cal A}_2$ recommends to ${\cal A}_1$ to add the attacks $(b, c)$ and $(c, a)$.
Note that no attacks are added between the previously existing arguments $a$ and $b$: $AF'$ is a \emph{normal expansion} of $AF$~\cite{baumann2010expanding}.
Given the attack relation $\{(a, b), (b, c), (c, a)\}$, many argumentation semantics resolve the arguments $\{a, b, c\}$ as $\{ \}$.
Let us refer to the constructed argumentation framework $(\{a, b, c\}, \{(a, b), (b, c), (c, a)\})$ as $AF'$.
If ${\cal A}_1$ adopts the recommended attack relation, this implies, according to the rational man paradigm, that it now prefers accepting  $\{\}$ over all elements in $2^{\{a, b, c\}}$, which in turn implies the preference $\{\} \succeq \{ a \}$.
This is inconsistent with the preference $\{ a \} \succeq \{ \}$ that ${\cal A}_1$ has given the set $\{a, b\}$ (see Figure~\ref{fig:examples}).
In this case, ${\cal A}_1$ can be considered to behave \emph{economically irrationally}: it has changed its preference order on the powerset of $\{a, b\}$, although \emph{i)} no attacks have been added between $a$ and $b$ and \emph{ii)} all arguments that have been added to the initial framework $AF$ are not considered valid conclusions and the attacks between existing arguments have not changed; \emph{i.e.}, the \emph{ceteris paribus} assumption (\textquote{all else unchanged}) holds true.
Figure~\ref{fig:examples} depicts the argumentation graphs of $AF$ and $AF'$.
\begin{figure}
    \subfloat[$AF$.]{
        \begin{tikzpicture}[
            noanode/.style={dashed, circle, draw=black!60, minimum size=10mm, font=\bfseries},
            anode/.style={circle, fill=lightgray, draw=black!60, minimum size=10mm, font=\bfseries},
            ]
            \node[anode]    (A)    at(0,4)  {a};
            \node[noanode]    (B)    at(0,2)  {b};
            \path [->, line width=1mm]  (A) edge node[left] {} (B);
        \end{tikzpicture}
    }
    \hspace{50pt}
    \centering
    \subfloat[$AF'$]{
        \begin{tikzpicture}[
            noanode/.style={dashed, circle, draw=black!60, minimum size=10mm, font=\bfseries},
            anode/.style={circle, fill=lightgray, draw=black!60, minimum size=10mm, font=\bfseries},
            ]
            \node[noanode]    (A)    at(0,4)  {a};
            \node[noanode]    (B)    at(0,2)  {b};
            \node[noanode]    (C)    at(2,2)  {c};
            \path [->, line width=1mm]  (A) edge node[left] {} (B);
            \path [->, line width=1mm]  (B) edge node[left] {} (C);
            \path [->, line width=1mm]  (C) edge node[left] {} (A);
        \end{tikzpicture}
    }
\caption{Inconsistent preferences: Given an argumentation semantics $\sigma$, such that $\sigma(AF) = \{\{a\}\}$ and $\sigma(AF') = \{\{\}\}$, $AF$ implies that accepting $\{ a \}$ is preferred over accepting $\{\}$, while $AF'$ implies accepting $\{\}$ is preferred over accepting $\{ a \}$.}
\label{fig:examples}
\end{figure}
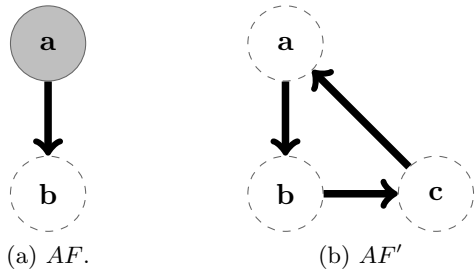
Let us note that we neither distinguish between different types of arguments (for example: utilitarian and epistemic arguments), nor do we assume that abstract argumentation frameworks model an exchange of natural language arguments.
In any case, the phenomenon of reference independence can be considered potentially problematic.
Given an agent that makes decisions on which utilitarian arguments in an argumentation framework should be considered valid conclusions, the analogy to economic theory and the decision-maker who has to choose a subset of a set of decision options is straight-forward: in both cases a utilitarian preference order is established.
But also given an agent that makes decisions on which epistemic arguments should be considered as true and false, it can be problematic if an agent changes its truth assessment of existing propositional atoms if the new atoms it considers are all disregarded\footnote{Note that in this context, we assume the agent wants to \textquote{commit} to a set of arguments it considers a valid conclusion. This is aligned with the \emph{clear preference} principle of economic decision-making (see: Rubinstein\cite[p.~7 et sqq.]{rubinstein1998modeling}).}, and the relations between existing atoms remain unchanged.
For example, given an agent has a knowledge base with the propositional atoms $a$ and $b$, and -- based on a binary relation between $a$ and $b$ -- considers $b$ as true and $a$ as false, it does not make sense that the addition of a new atom $c$ that the agent considers false changes the assessment of $b$ to false or of $a$ to true, as long as the relation between $a$ and $b$ remains unchanged.

In economics, the phenomenon that the presence of additional decision options changes existing preferences -- in this case $\{ a \} \succeq \{ \}$ to $\{ \} \succeq \{ a \}$ -- is called \emph{reference dependence}.
In this paper, we use the notion of preferences analogously, but with a broader intended meaning. An agent's preference over two sets $A \succeq B$ means that the agent considers the acceptance of the set $B$ preferable over the acceptance of the set $A$; this can imply the following, for example:
\begin{itemize}
    \item The agent prefers accepting a set of epistemic arguments $B$ as true over accepting a set of epistemic arguments $A$ as true.
    \item The agent prefers executing a set of actions $B$ over executing a set of actions $A$.
    \item The agent prefers executing a set of actions $B' \subseteq B$ and accepting a set of epistemic arguments $B'' = B \setminus B'$ as true over executing a set of actions $A' \subseteq A$ and considering a set of epistemic arguments $A'' = A \setminus A'$ as true. ($A$ and $B$ are sets of propositional atoms, and the agent interprets each propositional atom either as an action or as an epistemic argument.)
\end{itemize}
We show in this paper that changes
 in preferences can, given most argumentation semantics, be achieved without adding direct attacks between the arguments in an initial argumentation framework $AF$, \emph{i.e.}, by \emph{normally expanding} $AF$.
 
This means that in our scenario, ${\cal A}_2$ can attempt to deceive ${\cal A}_1$ into changing its preferences without noticing.
From the perspective of economic rationality, the existence of the \textquote{invalid} argument $c$ should not change the preferences ${\cal A}_1$ has already established given the arguments $\{a, b\}$.
Indeed, such a change of preferences can be problematic in practice.
For example, when an intelligent system needs to determine a set of actions, deciding $\{ a \}$ are valid conclusions given $\{a, b\}$, while deciding no arguments are valid conclusions ($\{ \}$) given $\{a, b, c\}$ does typically not make sense, given all other things remain the same (the so-called \emph{ceteris paribus} assumption in economic theory).

Note that we do not distinguish between arguments that are \textquote{out} and \textquote{undecided} (in contrast to many formal argumentation approaches, see, \emph{e.g.} Wu and Caminada for \emph{three-valued labeling}~\cite{wu2010labelling}).
A binary distinction of arguments that are either element of an extension or not is better aligned with the \emph{clear preferences} property that is inherent in the rational economic man paradigm (see Definition~\ref{clear-pref}) and that allows us to infer a clear preference relation from an argumentation framework, given a specific extension.
While we concede that a variant of \emph{boundedly} economically rational argumentation that considers the \textquote{undecided} status of arguments can be relevant future work, we consider this out of scope of this paper, analogously to formal argumentation approaches that qualify or quantify the uncertainty of arguments in other ways.
Let us highlight that with this approach, we are well-aligned with the problem in microeconomic theory to have a model that \textquote{include[s] a description of the \emph{resolution of all uncertainties} that influence [a] decision maker's interests}~\cite[p.~42]{rubinstein1998modeling}\footnote{Emphasis added by the authors}, as well as with the purely extension-based approach many other argumentation principles take (see Van der Torre and Vesic~\cite{van2017principle} for a survey).
Let us also note that we do not claim that the consistency of preference relations given an argumentation framework and its normal expansion is a principle that must be enforced given any application scenario.
Similarly to the general notion of economically rational decision-making, the practical usefulness of the principle may be use case-dependent. In this regard, this paper can be considered a first bridge between abstract argumentation and economically rational decision-making that can serve as a point of departure for more research, for example on argumentation-based models that can be of relevance to the economics community.

To analyze the problem of \emph{reference dependence} in abstract argumentation, this paper provides the following contributions:
\begin{enumerate}
    \item It characterizes the concepts of reference independence, cautious monotony, and rational monotony in the context of argumentation semantics and argumentation framework expansions, introducing \emph{strong} and \emph{weak} notions of each of the three concepts.
    \item It analyzes the relation between reference independence and monotony, \emph{cautious} \\ monotony, as well as \emph{rational} monotony. It proves that for semantics whose extensions are maximal conflict-free sets or maximal admissible sets (w.r.t. set inclusion), reference independence does not imply monotony but monotony implies reference independence. Reference independence does not imply cautious and rational monotony, and strong rational monotony does not imply strong reference independence, but weak rational monotony implies weak reference independence. However, the non-naive semantics that are evaluated in this paper all violate weak rational monotony, which indicates that weak rational monotony may be too strict to be useful as an argumentation principle. The paper also analyzes the relation between weak reference independence and directionality and SCC-recursiveness, and shows that weak reference independence neither implies directionality nor SCC-recursiveness (and vice versa).
    \item It proves that most well-established argumentation semantics cannot guarantee reference independence for normal argumentation framework expansions. A notably exception are CF2 semantics, which are \emph{weakly reference independent}.
    \item It proves that argumentation that is based on normal, non-cyclic expansions that also do not \textquote{add/change} attack sequences that originate from cycles -- which we refer to as \emph{rational man's expansions} -- guarantees reference independence for many argumentation semantics.
    \item It puts \emph{reference independence} and \emph{rational man's expansions} in the context of preference-based argumentation, an argumentation approach that explicitly models agent preferences. This shows that reference dependence is also a problem for preference-based argumentation, which can, however, in this case not be solved with the \emph{rational man's expansion} that this paper introduces.
    \item It introduces potential application examples for the presented theoretical results. An implementation of a dialogue reasoning engine that makes use of the theoretical results is presented in a short paper~\cite{kampik2020diarg}.
\end{enumerate}
The rest of this paper is organized as follows.
Section~\ref{ai-econ} grounds this paper in the context of research traditions at the intersection of economics and artificial intelligence, whereas Section~\ref{background} provides an overview of abstract argumentation frameworks, semantics and principles, as well as the relevant definitions of formal models of economic rationality.
Then, Section~\ref{inference} defines a set of non-monotonic reasoning properties in the context of abstract argumentation, before Section~\ref{preferences} presents argumentation frameworks in the context of decision functions.
Section~\ref{paradigms} establishes the argumentation principle the rational man paradigm requires. Hereby, the new rational man's \emph{reference independence} principle is defined, based on which Section~\ref{rational-semantics} shows that typically, with the notable exception of CF2 semantics (naive semantics are another exception), argumentation semantics do not guarantee reference independent decision-making.
To still allow for abstract argumentation in accordance with the rational man paradigm, a new argumentation framework \emph{expansion} that provides a naive approach for guaranteeing economic rationality by (severely) limiting the allowed expansions to normal, non-cyclic expansions that also do not \textquote{add} new attacks that originate from cycles, is defined in Section~\ref{rational-expansion}.
Subsequently, Section~\ref{dialoges} provides examples of how the established concepts can be applied in different scenarios.
Finally, Section~\ref{related} shows that preference-based argumentation cannot guarantee reference independent decision-making, either, before Section~\ref{conclusion} concludes the paper by highlighting relevant future work, in particular the importance of developing \textquote{loop-busting} approaches to allow for reference independent argumentation.

\section{Economic Rationality and Artificial Intelligence}\label{ai-econ}
Artificial intelligence research frequently draws from concepts that have first been established in economic theory.
For example, the notions of utility functions and preferences that are central to many models and algorithms of autonomous agents and multi-agent systems stem from ideas of the philosophers and economists Jeremy Bentham and John Stuart Mill~\cite{sen_1991}.
Von Neumann's and Morgenstern's ground-breaking game theoretical work \emph{Theory of Games and Economic Behavior}~\cite{von2007theory} has influenced generations of both economics and artificial intelligence researchers.
More recently, the works of Daniel Kahneman and Richard Thaler, who both have received the Nobel Memorial Prize in Economic Sciences for their research on behavioral economics, have inspired new work on and discussions about ethics and responsibility in artificial intelligence~\cite{Dignum:2018:EDN:3278721.3278745}.

A key concept at the intersection of economics and artificial intelligence is the notion of economic rationality.
In the same way that a rational economic decision-maker  (traditionally a human agent) is ideally expected to act according to clear and consistent preferences~\cite[p.~7 et sqq.]{rubinstein1998modeling}, artificially intelligent agents should \textquote{maximize [their] performance measure} as stated in the definition of a rational agent by Russel and Norvig~\cite[p.~37]{russell2016artificial}.
The concept of economical rationality provides a formal foundation for many ground-breaking works in economics, such as the Nash equilibrium~\cite{Nash48}, and Tversky's and Kahneman's empirical work on the limits of human agents to act rationally in the economic sense~\cite{kahneman2003maps}.

An important property in the context of economic rationality is \emph{reference independence}:
a decision-maker's preference order on a set of items $S$ should not be affected by the presence, or absence, of additional items $T$.
\emph{Reference dependence}, the negation of reference independence as established as a theory by Tversky, Slovic, and Kahneman~\cite{10.2307/2006743} and empirically validated as a phenomenon in human decision-making by, \emph{e.g.}, Bateman \emph{et al.}~\cite{bateman1997test}, stipulates that an agent might change their initial preference $a \succeq b$ out of the set of possible decision outomes $A$, with $a, b \in A$ to $b \succeq a$, depending on \emph{reference points} that are either added to the set of decision options itself or provided as additional context, but that do not impact the value or quality of either $a$ or $b$.
As a real-world example, let us summarize a study by Doyle \emph{et al.}~\cite{doyle1999robustness}.
In a grocery store, two brands of baked beans $x$ and $y$ are sold.
$x$ and $y$ are sold in cans of the same size.
Although $x$ is cheaper than $y$, only 19\% of bean sales are of brand $x$.
By adding a new option $x'$ -- a smaller can of beans by brand $x$ that is sold at the same prize as the original can $x$ -- the share of $x$ increases to 33\% of total sales (while sales of $x'$ are negligible).

From the body of research we summarize above it is clear that human decision-making is often not reference independent.
However, reference independence can potentially be considered a \emph{desired} property of the decision process of artificially intelligent agents, as we have demonstrated above, in the example that is visualized by Figure~\ref{fig:examples}.
Yet, when evaluating artificially intelligence systems, economic rationality in general, and reference dependence in particular, is typically not considered.
Instead, research results are commonly evaluated based on traditional performance criteria like accuracy and computational complexity.
In this context, a group of well-established artificial intelligence researchers advocates for a paradigm shift in the evaluation of \textquote{machine behavior}~\cite{Rahwan2019MachineB} that can be considered a continuation of the cross-disciplinary work of Herbert Simon, who initially coined the term \emph{bounded rationality}~\cite{10.2307/1884852}.
In the spirit of Simon, but also of Kahneman, who uses empirical methods to systematically specify the \emph{boundaries} of human rationality, this paper presents an application of the study of economic rationality to the area of formal argumentation.
As we study -- in contrast to Kahneman -- formally specified frameworks and not humans, we are able to use formal methods in our evaluations.
Let us highlight that the interface between formal models of economic rationality and formal argumentation that this paper establishes can potentially also help to apply formal argumentation as a method to solve problems in microeconomic theory.

\section{Theoretical Background}\label{background}
This section introduces the theoretical foundations upon which this paper builds:
the \emph{rational man} paradigm in economic theory and abstract argumentation.

\subsection{The Rational Economic Man Paradigm}
As a prerequisite, we introduce a definition of a partially ordered set~\cite{DavPri02}.
\begin{definition}
    Let ${\cal Q}$ be a set.
    An order (or partial order) on ${\cal Q}$ is a binary relation $\succeq$ on ${\cal Q}$ such that, for all $x,y,z
    \in {\cal Q}$:
    \begin{enumerate}
      \item $x \succeq x$ (reflexivity);
      \item $x \succeq y$ and $y \succeq x$ imply $x = y$ (antisymmetry);
      \item $x \succeq y$ and $y \succeq z$ imply $x \succeq z$ (transitivity).
    \end{enumerate}
\end{definition}
\noindent We refer to a set ${\cal Q}$ that is equipped with an order relation $\succeq$ as an \emph{ordered set} (or partially ordered set).

In economic theory, the model of a rational decision maker -- the \emph{rational economic man} paradigm\footnote{In this paper, we often use the shortened term \emph{rational man}.} -- can be described as follows (based on a definition by Rubinstein~\cite[p.~7 et sqq.]{rubinstein1998modeling})\footnote{However, in contrast to the model used by Rubinstein, we use a decision model that allows for choosing any subset $A^{*} \subseteq A$ of the set of decision options $A$ instead of exactly one option ($x^{*} \in A$).}.
\begin{definition}[Rational Economic Man]\label{clear-pref}
Given a set of decision options $A$, the rational economic man's decision $A^{*} \subseteq A$ implies the preference order (partial order) $\succeq$, such that $\forall A' \in 2^A, A^{*} \succeq A'$; 
\emph{i.e.}, $A^{*}$ is the rational man's preferred option when compared to all possible alternatives.
\end{definition}
For example, given the set $\{coffee, snack\}$, the rational man's preferred option could be $\{coffee, snack\}$, given he chooses this element from the set $2^{\{coffee, snack\}}$.
The decision implies the preference relation $\forall A' \in \{ \{ coffee \}, \{ snack \}, \{ \} \}, \{coffee, snack\} \succeq A'$, as depicted in Figure~\ref{fig:lattice-prefs} (note that technically also $\{coffee, snack\} \succeq \{coffee, snack\}$).
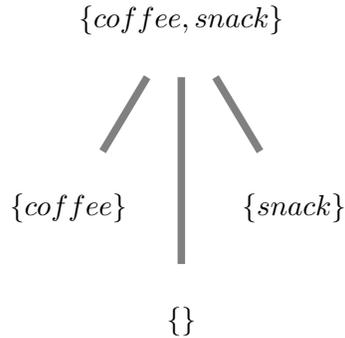
\begin{figure}
    \centering
    \pgfarrowsdeclare{arcs}{arcs}{...}
        {
        \pgfsetdash{}{0pt} 
        \pgfsetroundjoin   
        \pgfsetroundcap    
        \pgfpathmoveto{\pgfpoint{-10pt}{10pt}}
        \pgfpatharc{180}{270}{10pt}
        \pgfpatharc{90}{180}{10pt}
        \pgfusepathqstroke
        }
        \begin{tikzpicture}[
            noanode/.style={dashed, minimum size=15mm, font=\bfseries},
            anode/.style={ minimum size=15mm, font=\bfseries},
            ]
            \node[anode]    (All)    at(1.5,4)  {$\{coffee, snack\}$};
            \node[noanode]    (Coffee)    at(0,1.5)  {$\{coffee\}$};
            \node[noanode]    (Snack)    at(3,1.5)  {$\{snack\}$};
            \node[noanode]    (None)    at(1.5,0)  {$\{\}$};
            \path [line width=1mm, gray]  (All) edge node[left] {} (None);
            \path [line width=1mm, gray]  (All) edge node[left] {} (Coffee);
            \path [line width=1mm, gray]  (All) edge node[left] {} (Snack);
        \end{tikzpicture}
        \caption{Example: Hasse diagram of a set, in which the following \emph{preference relation} of the rational man holds: $\forall A' \in \{ \{ coffee \}, \{ snack \}, \{ \} \}, \{coffee, snack\} \succeq A'$.}
        \label{fig:lattice-prefs}
\end{figure}

In the context of this paper, an important property of the rational man's decision process is \emph{reference independence}.
\begin{definition}[Reference Independence]\label{ref-in}
Given two sets of options $A$ and $A'$, such that $A \subseteq A'$  and the rational man's decisions $A^{*} \subseteq A$ and $A'^{*} \subseteq A'$, if $A'^{*} \subseteq A$, then $A'^{*} = A^{*}$.
\end{definition}
Colloquially expressed, the rational man's decision outcome is not affected if options that the agent does not prefer over the \textquote{previously} existing options are added to the set of potential options.
\begin{example}
    If a decision-maker, given the set $\{tea, cookie\}$ chooses $\{tea\}$ and given the set $\{tea, cookie, \\coffee\}$ chooses $\{tea, cookie\}$ the decision is not rational; the presence of the \textquote{new} \emph{irrelevant alternatives} $\{coffee\}$, $\{tea, coffee\}$, $\{cookie, coffee\}$, and $\{tea, cookie, coffee\}$ causes the preference $\{tea\} \succeq \{tea, cookie\}$ to change to $\{tea, cookie\} \succeq \{tea\}$.
\end{example}
The \emph{reference independence} property implies that the expansion of a set $A$ to $A' \supset A$ does not add new knowledge about the elements in $A$ that affects the decision-maker's preferences over any two elements in $2^{A}$.
Yet, in reality, human decision-makers can use the addition of options to a set as a way to infer new information about the quality of the original options, as for example shown by Doyle \emph{et al.}~\cite{doyle1999robustness}.
Still, in the context of computational argumentation-based decision-making, it can be useful to be able to distinguish between changes in options and outcomes that satisfy the \emph{reference independence} property and those that do not.
For example, in a consultation scenario, a consulted agent may want to check if the consulting agent is providing a proposal that is \emph{seemingly} barely adding new options, but covertly also altering the preference order on existing options.
Such a scenario is presented in greater detail in Section~\ref{dialoges}.

\subsection{Abstract Argumentation}\label{argumentation}
To allow for a concise overview of the relevant foundations of abstract argumentation, we introduce a formal definition of the basic structure of an Argumentation Framework (AF), provide the definitions of well-established argumentation semantics, and explain the notions of argumentation expansions and principles.
\begin{definition}[Argumentation Framework~\cite{dung1995acceptability}]\label{arg-f}
An argumentation framework is a pair $AF = (AR, Attacks)$, where $AR$ is a finite set of arguments, and $Attacks$ is a binary relation on $AR$, \emph{i.e.}, $Attacks \subseteq AR \times AR$.
\end{definition}
Note that we diverge from the original definition by Dung in that we constrain an argumentation framework to a finite set of arguments.
In the context of an argumentation framework $AF = (AR, Attacks)$, \emph{a attacks b} means that for $a, b \in AR$ it holds true that $(a, b) \in Attacks$. Similarly, any set \emph{$S \subseteq AR$ attacks $b$} if $b$ is attacked by an argument in $S$. We say that an argument $c \in AR$ is \emph{unattacked} iff $\nexists d \in AR$, such that $d$ attacks $c$ and we say that an argument $g \in AR$ is \emph{self-attacking} iff $g$ attacks $g$.
\\

\noindent To select coherent sets of arguments from an argumentation framework, Dung introduces the notion of \emph{conflict-free} and \emph{admissible} sets; based on the notion of admissible sets, Baroni \emph{et al.} define \emph{strongly admissible} sets.
\begin{definition}[Conflict-free, Acceptable, and (Strongly) Admissible Arguments~\cite{dung1995acceptability,BARONI2007675}]
Let $AF = (AR, Attacks)$ be an argumentation framework.
\begin{itemize}
    \item A set $S \subseteq AR$ is \emph{conflict-free} iff there are no arguments $a, b$ in $S$ such that $a$ attacks $b$.
    \item An  argument $a \in AR$ is \emph{acceptable} w.r.t. an argument set $S \subseteq AR$ iff for each argument $b \in AR$ it holds true that if $b$ attacks $a$, then $b$ is attacked by $S$.
    \item A conflict-free set $S \subseteq AR$ is \emph{admissible} iff each argument in $S$ is acceptable with regards to $S$.
    \item A set $S \subseteq AR$ is strongly admissible iff $\forall c \in S$, there exists a set $S_{sub} \subseteq S \setminus \{c\}$, such that $c$ is acceptable with regards to $S_{sub}$ and $S_{sub}$ is strongly admissible.
\end{itemize}
\end{definition}
Based on these concepts, Dung and others define different \emph{semantics} for resolving argumentation frameworks.
\begin{definition}[Argumentation Semantics~\cite{dung1995acceptability}]\label{semantics-def} 
An argumentation semantics $\sigma$ is a function that takes an argumentation framework $AF = (AR, Attacks)$ and returns a set of sets of arguments denoted by $\sigma(AF)$.
\begin{align*}
    \sigma: {\cal AF} \rightarrow 2^{2^{AR}},
\end{align*}
with ${\cal AF}$ denoting the set of all possible argumentation frameworks.
Each set of $\sigma(AF)$ is called a $\sigma$-extension.
\end{definition}
Some argumentation semantics make use of the notion of \emph{range}.
\begin{definition}[Range~\cite{verheij1996two}]
    Let $AF = (AR, Attacks)$ be an argumentation framework and let $S \subseteq AR$.
    We define the \emph{range} of $S$ as $S \cup S^+$,
    where $S^+ =  \{b|b \in AR,  a \in S, (a, b) \in Attacks\}$.
\end{definition}
Below, we provide an overview of some of the semantics that are commonly applied and analyzed by the formal argumentation community.
\begin{definition}[Argumentation Extensions and Semantics~\cite{dung1995acceptability,DungMT07,caminada2012semi,caminada2007comparing,verheij1996two,baroni2018abstract}\label{arg-extensions}]
Given an argumentation framework $AF = (AR, Attacks)$, an \emph{admissible} set $S \subseteq AR$ is a:  
    \begin{itemize}
        \item \textbf{Stable extension} of $AF$ iff $S$ attacks each argument that does not belong to $S$. Stable semantics $\sigma_{stable}(AF)$ returns all stable extensions of $AF$.
        \item \textbf{Preferred extension} of $AF$ iff $S$ is maximal (w.r.t. set inclusion) among admissible sets. Preferred semantics $\sigma_{preferred}(AF)$ returns all preferred extensions of $AF$.
        \item \textbf{Complete extension} of $AF$ iff each argument that is acceptable w.r.t. $S$ belongs to $S$. Complete semantics $\sigma_{complete}(AF)$ returns all complete extensions of $AF$.
        \item \textbf{Grounded extension} of $AF$ iff $S$ is the minimal (w.r.t. set inclusion) complete extension.
        Grounded semantics $\sigma_{grounded}(AF)$ denotes the set $\{S\}$, such that $S$ is the grounded extension of $AF$.
        \item \textbf{Ideal extension} of $AF$ iff $S$ is the maximal (w.r.t. set inclusion) ideal set, where a set $S' \subseteq AR$ is \emph{ideal} iff $S'$ is a subset of every preferred extension of $AF$. Ideal semantics $\sigma_{ideal}(AF)$ denotes the set $\{S\}$, such that $S$ is the ideal extension of $AF$.
        \item \textbf{Semi-stable extension} of $AF$ iff $S$ is a complete extension where $S \cup S^+$ is maximal (w.r.t. set inclusion) among complete extensions. Semi-stable semantics $\sigma_{semi-stable}(AF)$ returns all semi-stable extensions of $AF$.
        \item \textbf{Eager extension} of $AF$ iff $S$ is the maximal (w.r.t. set inclusion) admissible set which is included in every semi-stable extension. Eager semantics $\sigma_{eager}(AF)$ denotes the set $\{S\}$, such that $S$ is the eager extension of $AF$.
    \end{itemize}
    Given an argumentation framework $AF = (AR, Attacks)$, a set $S \subseteq AR$ is a:
    \begin{itemize}
        \item \textbf{Naive extension} iff $S$ is a maximal (w.r.t. set inclusion) conflict-free set. Naive semantics $\sigma_{naive}(AF)$ returns all naive extensions of $AF$.
        \item \textbf{Stage extension} of $AF$, iff $S$ is conflict-free and $S \cup S^+$ is maximal (w.r.t. set inclusion) among all conflict-free sets, \emph{i.e.}, there is no conflict-free set $S' \subseteq AR$, such that $(S' \cup S'^{+}) \supset (S \cup S^{+})$. Stage semantics $\sigma_{stage}(AF)$ returns all stage extensions of $AF$.
    \end{itemize}
    \end{definition}
    To introduce the definitions of \emph{stage2} and \emph{CF2} semantics, let us first define the notion of \emph{attack sequences} and \emph{strongly connected components}.
    \begin{definition}[Attack Sequences in Argumentation Frameworks]~\label{attack-seq}
    An attack sequence $V_{a_1, a_n}$ in an argumentation framework $AF = (AR, Attacks)$ is an argument sequence $\langle a_1, ... , a_n \rangle$ such that $(a_{i}, a_{i+1}) \in Attacks$, $a_i \in AR$, $a_i \neq a_j, 1 \leq j \le i$, $1 \leq i \leq n-1$. ${\cal V}(AF)$ denotes all attack sequences of AF. 
    Given an attack sequence $V_{a_1, a_n}$, $AR^{V}_{a_1, a_n}$ denotes all arguments that occur in $V$.
    An argument $b$ is reachable from an argument $a$ iff there exists an attack sequence $V_{a_1, a_n} \in {\cal V}(AF)$, such that $a_1 = a$ and $a_n = b$.
    \end{definition}
    Strongly connected components (SCCS) \textquote{provide a unique partition of a directed graph into disjoint parts where all nodes are mutually reachable}~\cite{baroni2018abstract}.
    \begin{definition}[Strongly Connected Components]
    Let $AF = (AR, Attacks)$ be an argumentation framework. $S \subseteq AR$ is a strongly connected component of $AF$, iff $\forall a, b \in S$, $a$ is reachable from $b$ and $b$ is reachable from $a$ and $\nexists c \in AR \setminus S$, such that $a$ is reachable from $c$ and $c$ is reachable from $a$.
    Let us denote the strongly connected components of $AF$ as $SCCS_{AF}$.
    \end{definition}    
    For example, the graph of the argumentation framework $(\{a, b, c\}, \{(a, b), (b, a), (b, c)\})$ has the SCCS $\{a, b\}$ and $\{ c \}$.
    
    Let us also define the concept of a \emph{restriction} of an argumentation framework to a subset of its arguments.
    \begin{definition}[Argumentation Framework Restriction~\cite{baroni2018abstract}]
        Let $AF = (AR, Attacks)$ be an argumentation framework and let $S \subseteq AR$ be a set of arguments. The restriction of $AF$ to $S$, denoted by $AF\downarrow_{S}$, is the argumentation framework $(AR \cap S, Attacks \cap (S \times S))$.
    \end{definition}
    Let us introduce some additional preliminaries.
\begin{definition}[D, P, U, and UP~\cite{BARONI2005162}]
    Let $AF = (AR, Attacks)$ be an argumentation framework, let $E \subseteq AR$, and let $S$ be a strongly connected component of $AF$ ($S \in SCCS(AF)$).
    Let us define $S_{out}^{-} = \{a | a \in AR$, $a \not \in S$ and $a$ attacks $S\}$.
    Given an argument $a \in AR$, let us define $a^{-} = \{b | b \in AR$ and $a$ attacks $b\}$.
    We define:
    \begin{itemize}
        \item $D_{AF}(S, E) = \{a | a \in S, (E \cap S_{out}^{-})$ attacks $a\}$;
        \item $P_{AF}(S, E) = \{a | a \in S, (E \cap S_{out}^{-})$ does not attack a and $\exists b \in (S_{out}^{-} \cap a^{-})$ such that $E$ does not attack $b\}$;
        \item $U_{AF}(S, E) = S \setminus (D_{AF}(S, E) \cup P_{AF}(S, E))$;
        \item $UP_{AF}(S, E) = U_{AF}(S, E) \cup P_{AF}(S, E)$.
    \end{itemize}
\end{definition}
    Now, we have the necessary concepts established that allow us to define SCC-recursive extensions.

    \begin{definition}[CF2 Extensions~\cite{BARONI2005162}]\label{cf-2}
        Let $AF = (AR, Attacks)$ be an argumentation framework.
        A set of arguments $S \subseteq AR$ is an extension of CF2 semantics iff:
    \begin{itemize}
        \item $S$ is a naive extension of $AF$, if $|SCCS_{AF}| = 1$;
        \item $\forall Args \in SCSS_{AF}, (S \cap Args) \in \sigma_{CF2}(AF\downarrow_{UP_{AF}(Args,S)})$, otherwise.
    \end{itemize}
    CF2 semantics $\sigma_{CF2}(AF)$ returns all CF2 extensions of $AF$.
    \end{definition}

    \begin{definition}[Stage2 Extensions~\cite{10.1093/logcom/exu006}]\label{stage-2}
        Let $AF = (AR, Attacks)$ be an argumentation framework. A set $S \subseteq AR$ is an extensions of stage2 semantics, iff:
        \begin{itemize}
            \item $S$ is a stage extension of $AF$, if $|SCCS_{AF}| = 1$;
            \item $\forall Args \in SCSS_{AF}, (S \cap Args) \in \sigma_{stage2}(AF\downarrow_{UP_{AF}(Args,S)})$, otherwise.
        \end{itemize}
        Stage2 semantics $\sigma_{stage2}(AF)$ returns all stage2 extensions of $AF$.
    \end{definition}
In words, CF2 and stage2 semantics decompose an argumentation framework into SCCS and recursively resolve the framework component by component, starting with the SCCS that are not attacked by any other SCC.
Note that a detailed explanation of CF2 and stage2 semantics is beyond the scope of the paper.

Finally, we provide the established definitions for \emph{credulous} and \emph{skeptical} acceptance of arguments~\cite{10.1007/978-3-540-30227-8_39}\footnote{From a decision-making perspective, \emph{credulous and skeptical acceptance} are often referred to as \emph{objective and subjective acceptance}~\cite{bench2003persuasion}; however, we use the terms \emph{credulous and skeptical} because they are arguably more prevalent in the community and because our work is relevant no matter the type of reasoning (making a decision about an action to take or deciding which beliefs in a set should be considered as true).
}.

\begin{definition}[Credulous and Skeptical Reasoning Modes]
    Given an argumentation framework $AF = (AR, Attacks)$ and an argumentation semantics $\sigma$ we say that:
    \begin{itemize}
        \item an argument $a \in AR$  is \textbf{credulously accepted} w.r.t. $\sigma$ iff it is contained in \emph{at least one} of the extensions of $\sigma(AF)$;
        \item an argument $a \in AR$  is \textbf{skeptically accepted} w.r.t. $\sigma$ iff it is contained in \emph{all} of the extensions of $\sigma(AF)$. The skeptical extension of an argumentation framework $AF$ w.r.t. $\sigma$ is the intersection of all extensions of $\sigma(AF)$ ($\sigma^{\cap}(AF) = \{\bigcap_{E \in \sigma(AF)} E \}$).
    \end{itemize}
\end{definition}
Note that to be consistent with the general function signature of argumentation semantics, we provide the aggregation as the intersection of sets as made by a semantics' skeptical reasoning mode wrapped into an additional set.

Given the variety of argumentation semantics that have been established throughout the years, it can be challenging to assess which semantics are suitable for a specific application.
Consequently, \emph{argumentation principles} have been defined; evaluating the satisfaction of one or multiple principles can guide the assessment of argumentation semantics in the context of a specific use case. 
For example, a principle can specify that an argumentation semantics should determine exactly one extension for every argumentation framework.
An overview of argumentation principles is provided by Van der Torre and Vesic~\cite{van2017principle}.

\subsection{Argumentation Expansions}
To describe a specific type of relation between argumentation frameworks, the notion of an argumentation framework \emph{expansion} has been introduced by Baumann and Brewka~\cite{baumann2010expanding}.
A concise introduction and overview of expansions is presented by Baumann and Woltran~\cite{8209878}.
The general concept of an argumentation framework expansion can be defined as follows.
\begin{definition}[Argumentation Framework Expansion~\cite{baumann2010expanding}] 
    An argumentation framework $AF' = (AR', Attacks')$ is an expansion of another argumentation framework $AF = (AR, Attacks)$ (denoted by $AF \preceq_E AF'$) iff $AR \subseteq AR'$ and $Attacks \subseteq Attacks'$.
\end{definition}
Several expansion types have been defined in the literature.
In the context of this paper, the notion of a \emph{normal} expansion is of relevance.
\begin{definition}[Normal Expansion~\cite{baumann2010expanding}] 
   An argumentation framework $AF' = (AR', Attacks')$ is a normal expansion of an argumentation framework $AF = (AR, Attacks)$ (denoted by $AF \preceq_N AF'$) iff $AF \preceq_E AF'$ and $\nexists (a, b) \in Attacks' \setminus Attacks$, such that $a \in AR \land b \in AR$.
\end{definition}
In words, a normal expansion of an argumentation framework adds additional arguments to the framework that can attack and be attacked by any other argument but neither removes existing arguments nor changes attacks (neither adds nor removes) between existing arguments.

\section{Properties of Non-monotonic Inference in Argumentation}
\label{inference}
Previous research exists that explores the relation of non-monotonic inference properties and formal argumentation methods.
In particular, \v{C}yras and Toni provide a series of proofs that show which semantics (among grounded, ideal, skeptically preferred, stable, preferred, and complete semantics) satisfy the \emph{cautions monotony} properties\footnote{\emph{Cautions monotony} is based on a property initially introduced by Gabbay as \emph{restricted monotony}~\cite{10.1007/978-3-642-82453-1_15}.} for assumption-based argumentation~\cite{10.1007/978-3-319-28460-6_6}.
However, \v{C}yras' and Toni's definition cannot be applied to abstract argumentation as it relies on specific characteristics of assumption-based argumentation.
We use the same notions of \emph{strong} and \emph{weak} monotony properties as introduced by \v{C}yras and Toni to characterize monotony in the context of abstract argumentation.
Let us note that we define the properties in the context of \emph{normal expansions} because changes to attacks between existing arguments will obviously cause problems for most semantics w.r.t. compliance with the properties.
As a preliminary, let us introduce the definition of universally defined semantics, in our case given finite argumentation semantics (see: Baumann~\cite{baumann2017nature}).

\begin{definition}[Universally Defined Semantics]
 An argumentation semantics $\sigma$ is universally defined iff for every argumentation framework $AF = (AR, Attacks)$ it holds true that $|\sigma(AF)| \geq 1$.
\end{definition}
Let us now define monotony as an argumentation principle.
\begin{definition}[Monotony in Argumentation Semantics]
Let $\sigma$ be an argumentation semantics.
$\sigma$ satisfies:
    \begin{itemize}
        \item \textbf{Strong monotony} iff $\sigma$ is universally defined and for every two argumentation frameworks $AF = (AR, Attacks)$ and $AF' = (AR', Attacks')$, such that $AF \preceq_N AF'$, $\forall E \in \sigma(AF), \forall E' \in \sigma(AF')$ it holds true that $E \subseteq E'$.
        \item \textbf{Weak monotony} iff for every two argumentation frameworks $AF = (AR, Attacks)$ and $AF' = (AR', Attacks')$, such that $AF \preceq_N AF'$, $\forall E \in \sigma(AF)$, it holds true that $\exists E' \in \sigma(AF')$, such that $E \subseteq E'$.
    \end{itemize}
\end{definition}
Now, let us define cautious monotony, based on a definition provided (\emph{e.g.}) by Schröder \emph{et al.}~\cite{10.5555/1860967.1861106}, \emph{i.e.}, $\textit{if } (a \vsim b) \textit{ and } (a \vsim c) \textit{ then }(a \land b \vsim c)$.
In words, the property can be described as \emph{if we infer $c$ from $a$, and we infer $b$ from $a$, then we infer $c$ from \textquote{$a$ and $b$}}.
In the context of abstract argumentation we can say that an argumentation semantics \textquote{infers sets of arguments as parts of extensions from argumentation frameworks}. Using an argumentation semantics to infer the sets of arguments $c$ and $b$ from an argumentation framework $a$ implies that $b$ is a subset of the arguments in argumentation framework $a$; \emph{i.e.}, a straight-forward adaptation of the cautious monotony principle is not useful in the case of abstract argumentation (this problem does not exist in assumption-based argumentation, see \v{C}yras and Toni~\cite{10.1007/978-3-319-28460-6_6}).
Instead we state that:
\begin{itemize}
    \item Given a semantics $\sigma$, if we infer a set of arguments $c$ from an argumentation framework $a$, and
    \item If normally expanding $a$ by the arguments and attacks $b'$ to an argumentation framework $b$ and arguments in $b'$ do not attack $c$,
    \item Then we can infer $c$ from $b$.
\end{itemize}

We can consequently define cautious monotony for argumentation semantics as follows.
\begin{definition}[Cautious Monotony in Argumentation Semantics]
    Let $\sigma$ be an argumentation semantics.
$\sigma$ satisfies:
    \begin{itemize}
        \item \textbf{Strong cautious monotony} iff $\sigma$ is universally defined and for every two argumentation frameworks $AF = (AR, Attacks)$ and $AF' = (AR', Attacks')$, such that $AF \preceq_N AF'$, $\forall E \in \sigma(AF)$ and $\forall E' \in \sigma(AF')$ it holds true that if $\{(a,b) \mid (a,b) \in Attacks', a \in AR' \setminus AR, b \in E \}=\emptyset$ then $E \subseteq E'$;
        \item \textbf{Weak cautious monotony}: iff for every two argumentation frameworks $AF = (AR, Attacks)$ and $AF' = (AR', Attacks')$, such that $AF \preceq_N AF'$ and $\forall E \in \sigma(AF)$, it holds true that if $\{(a,b) \mid (a,b) \in Attacks', a \in AR' \setminus AR, b \in E \}=\emptyset$ then $\exists E' \in \sigma(AF')$ such that $E \subseteq E'$.
    \end{itemize}
\end{definition}

A property that has so far not been used to analyze abstract argumentation approaches is \emph{rational monotony}, as for example provided by Benferhat \emph{et al.} as $\textit{if } \neg  (a \vsim \neg b) \textit{ and } (a \vsim c) \textit{ then } (a \land b \vsim c)$~\cite{10.1016/S0004-3702(97)00012-X}.
In words, the property can be described as \emph{if we infer $c$ from $a$, and we do not infer \textquote{$\neg b$} from $a$, then we infer $c$ from \textquote{$a$ and $b$}}.
Analogously to cautious monotony, which requires an adaptation to be useful in the context of abstract argumentation, we need to \textquote{tweak} the $\neg  (a \vsim \neg b)$ condition and we state that:
\begin{itemize}
    \item Given a semantics $\sigma$, if we infer a set of arguments $c$ from an argumentation framework $a$, and
    \item If normally expanding $a$ by the arguments and attacks $b'$ to an argumentation framework $b$ and arguments in $b'$ that are in any $\sigma$-extension of $b$ do not attack $c$,
    \item Then we can infer $c$ from $b$.
\end{itemize}
\begin{definition}[Rational Monotony in Argumentation Semantics] 
         Let $\sigma$ be an argumentation semantics.
        $\sigma$ satisfies:
        \begin{itemize}
            \item \textbf{Strong rational monotony}: iff $\sigma$ is universally defined and for every two argumentation frameworks $AF = (AR, Attacks)$ and $AF' = (AR', Attacks')$, such that $AF \preceq_N AF'$, $\forall E \in \sigma(AF)$ and $\forall E' \in \sigma(AF')$ it holds true that if $\{(a,b) \mid (a,b) \in Attacks', a \in UE'_{new}, b \in E \}=\emptyset$ then $E \subseteq E'$, where $UE'_{new} = \bigcup_{E' \in \sigma(AF')} (E' \setminus AR)$;
            \item \textbf{Weak rational monotony}: iff for every two argumentation frameworks $AF = (AR, Attacks)$ and $AF' = (AR', Attacks')$, such that $AF \preceq_N AF'$ and $\forall E \in \sigma(AF)$, it holds true that if $\{(a,b) \mid (a,b) \in Attacks', a \in  UE'_{new}, b \in E \}=\emptyset$ then $\exists E' \in \sigma(AF')$ such that $E \subseteq E'$, where $UE'_{new} = \bigcup_{E' \in \sigma(AF')} (E' \setminus AR)$.
        \end{itemize}
        
\end{definition}

Let us observe that when considering the two principles of non-monotonic reasoning, rational monotony implies cautious monotony:
\begin{itemize}
    \item Rational monotony: $\: \textit{if } \neg  (a \vsim \neg b) \textit{ and } (a \vsim c) \textit{ then } (a \land b \vsim c)$.
    \item Cautious monotony: $\textit{if } \:\:\: (a \vsim  \:\:\: b) \textit{ and } (a \vsim c) \textit{ then }(a \land b \vsim c)$.
\end{itemize}
We can see that the if-condition of rational monotony is weaker, \emph{i.e.}, $(a \vsim b)$ implies $\neg  (a \vsim \neg b)$, but not vice versa, from which it follows that the principle of rational monotony is stricter.
Analogously, the rational monotony argumentation principle implies the cautious monotony argumentation principle, assuming that an argumentation semantics, given any argumentation framework $AF = (AR, Attacks)$, and $E \in \sigma(AF)$, $\forall AF' = (AR', Attacks')$, such that $AR' = AR$ and $Attacks \subseteq Attacks'$, $Attacks' \supseteq (Attacks \setminus \{(a, b) | (a, b) \in AR', b \in E\})$, it holds true that $\exists E' \in \sigma(AF')$, such that $E \subseteq E'$.
Colloquially speaking, we assume that \textquote{removing} attacks from arguments in $AR$ to an extension $E$ does not cause any argument in $E$ to be no longer considered part of a valid conclusion (extension).
This follows from the fact that the only difference between the rational monotony and cautious monotony argumentation principles is that the constraints of rational monotony, given an argumentation semantics $\sigma$ and two argumentation frameworks $AF = (AR, Attacks)$, $AF' = (AF', Attacks')$, such that $AF \preceq_N AF'$, are weaker:
\begin{itemize}
    \item Rational monotony: \\
    $\{(a,b) \mid (a,b) \in Attacks', a \in \bigcup_{E' \in \sigma(AF')} (E' \setminus AR), b \in E \}=\emptyset$.
    \item Cautious monotony: \\ 
    $\{(a,b) \mid (a,b) \in Attacks', a \in AR' \setminus AR, b \in E \}=\emptyset$.
\end{itemize}

Considering the paper's primary focus on the intersection of economic rationality and abstract argumentation, a detailed analysis of cautious and rational monotony can be considered out-of-scope.
The purpose of the establishment of the cautious monotony and rational monotony argumentation principles is to show the difference between these variants of \textquote{relaxed} monotony and the reference independence argumentation principle we derive from a formal model of economic rationality.

\section{Rational Argumentation-based Decision Functions}
\label{preferences}
In this section, we define the concept of \emph{economically rational argumentation-based decision functions}.

\subsection{Rational Argumentation-based Decision-Making}
To build the foundation for exploring the intersection of abstract argumentation semantics and economic rationality, we introduce the notion of an argumentation-based decision function.
\begin{definition}[Argumentation-based Decision Function]\label{arg-function}
    \label{preference-function}
    The argumentation-based decision function $g \circ \sigma$ of an agent is the function composition between a function $g$ and an argumentation semantics $\sigma$ that takes an argumentation framework $AF = (AR, Attacks)$ as its input and returns a set of decision outcomes $AR^{*} \subseteq AR$:
    \begin{align*}
        g \circ \sigma: {\cal AF} \rightarrow 2^{AR}
    \end{align*}
\end{definition}
Because we assume an economically rational decision function, given an argumentation framework $AF = (AR, Attacks)$, $A^{*} =  g \circ \sigma(AF)$ implies $\forall A \in 2^{AR}, A^{*} \succeq A$ (\emph{i.e.}, the agent establishes the preference relation $A^{*} \succeq A$ for all $A \in 2^{AR}$).

To determine the function's input, an economically rational agent needs to construct an argumentation framework $AF$ that consists of the propositional atoms $AR$ and an attack relation $Attacks \subseteq AR \times AR$.
The framework is resolved by an argumentation semantics $\sigma$; different argumentation semantics can be used.
In the context of rational economic decision-making, we want to determine exactly one set of arguments that can be considered valid conclusions.
To achieve this, the decision function can be defined, for instance, as $g^{\cup} \circ \sigma(AF) = \bigcup_{E \in \sigma(AF)} E$; \emph{i.e.}, if a semantics $\sigma$ returns more than one extension, the decision function $g^{\cup} \circ \sigma (AF) = \bigcup_{E \in \sigma(AF)} E$ returns the union of the set of extensions returned by $\sigma$. We call such a function a \emph{lenient} decision function.
For example, given a particular argumentation semantics, we can define argumentation-based decision-functions as follows:
\begin{enumerate}
    \item $g^{\cap} \circ \sigma_{stage}(AF)$ corresponds to the skeptical reasoning mode of stage semantics.
    \item $g^{\cap} \circ \sigma_{grounded}(AF)$ and $g^{\cup} \circ \sigma_{grounded}(AF)$ return the grounded extension of $AF$. There is always exactly one grounded extension, from which it follows that $g^{\cap} \circ \sigma_{grounded}(AF) = g^{\cup} \circ \sigma_{grounded}(AF)$. The same applies to ideal and eager semantics.
\end{enumerate}
Other approaches to use $g$ as an aggregate function for \textquote{selecting} an extension from the set of extensions a semantics returns are possible, of course.

It is important to note that argumentation-based decision-making does not necessarily imply choice from a set of goods or indeed any type of scenario that is typical for economic decision-making examples, but can cover any decision process, for example the selection of epistemic arguments from an argumentation framework that an agent will consider as valid.

\begin{example}
    As an example, let us assume we have a consultant who has to suggest whether or not to launch a product $p$.
    Initially, she does not find any arguments against launching the product, \emph{i.e.}, she establishes the argumentation framework $AF = (\{l_p\}, \{\})$, where $l_p$ stands for \textquote{launch product}.
    When asked by a decision-maker about whether to launch the product or not, our consultant can give a clear recommendation given any of the common argumentation semantics, let us assume, for example, stage semantics $\sigma_{stage}$ or preferred semantics $\sigma_{preferred}$: $\sigma_{stage}(AF) = \sigma_{preferred}(AF) = \{\{l_p\}\}$.
    From the perspective of microeconomic decision theory, we can infer that this result implies the consultant has established the preference order that $\{l_p\}$ is preferred over all other items in $2^{\{l_p\}} = \{\{\}, \{l_p\}\}$, \emph{i.e.}, $\forall A \in \{\{\}, \{l_p\}\}, \{l_p\} \succeq A$.
    However, let us now assume that the decision-maker asks the consultant to collect feedback from different management stakeholders on whether or not the product should be launched.
    After doing so, the consultant constructs an expansion of $AF$ that reflects the different arguments for and against launching the product (directly or indirectly). She ends up with the argumentation framework $AF' = (\{l_p, a, b, c\},\{(a, l_p), (a, b), (b, c), (c, a)\})$.
    Let us assume she again uses either preferred or stage semantics to determine the framework's extensions:
    \begin{itemize}
        \item Preferred semantics. $\sigma_{preferred}(AF) = \{\{\}\}$. This means that the addition of the arguments $a$, $b$, and $c$ changes the status of $l_p$ although $a$, $b$, and $c$ are all not considered valid conclusions and although the relationships between the arguments in $AF$ have not changed; \emph{i.e.}, the consultant reverses a preference from $\{l_p\} \preceq \{\}$ to $\{\} \preceq \{l_p\}$ and is economically not rational.
        Of course, the consultant could search for new knowledge until she can make an economically rational decision (see work on \textquote{loop-busting} in abstract argumentation as presented by Gabbay~\cite{10.1093/logcom/exu007}), but preferred semantics do not allow her to commit to an economically rational decision at the moment.
        \item Stage semantics.
        $\sigma_{stage}(AF) = \{\{l_p, b\}, \{l_p, c\}, \{a\}\}$.
        This result means that the consultant can \textquote{pick} an extension (either $\{l_p, b\}$ or $\{l_p, c\}$, or $\{a\}$) that implies preferences that are consistent with the previous recommendation based on $AF$.
        This can be considered a useful property in this scenario, where we can assume that the decision-maker expects a clear recommendation on whether the product should be launched or not that is consistent with regards to the previous interactions between decision-maker and consultant, and indecisiveness is most likely not an option.
    \end{itemize}
    Figure~\ref{fig:example2} depicts the example's argumentation graphs.
\end{example}

\begin{figure}
    \subfloat[$AF$.]{
        \begin{tikzpicture}[
            noanode/.style={dashed, circle, draw=black!60, minimum size=10mm, font=\bfseries},
            anode/.style={circle, fill=lightgray, draw=black!60, minimum size=10mm, font=\bfseries},
            ]
            \node[anode]    (lp)    at(0,4)  {l$_{\text{p}}$};
        \end{tikzpicture}
    }
    \hspace{50pt}
    \centering
    \subfloat[$AF'$]{
        \begin{tikzpicture}[
            noanode/.style={dashed, circle, draw=black!60, minimum size=10mm, font=\bfseries},
            unode/.style={circle, draw=black!60, minimum size=10mm, font=\bfseries},
            anode/.style={circle, fill=lightgray, draw=black!60, minimum size=10mm, font=\bfseries},
            ]
           \node[unode]    (lp)    at(0,4)  {l$_{\text{p}}$};
            \node[unode]    (a)    at(0,2)  {a};
            \node[unode]    (b)    at(2,2)  {b};
             \node[unode]    (c)    at(2,4)  {c};
            \path [->, line width=1mm]  (a) edge node[left] {} (lp);
            \path [->, line width=1mm]  (a) edge node[left] {} (b);
            \path [->, line width=1mm]  (b) edge node[left] {} (c);
            \path [->, line width=1mm]  (c) edge node[left] {} (a);
        \end{tikzpicture}
    }
\caption{Inconsistent preferences, given preferred semantics: $\sigma_{preferred}(AF) = \{\{l_{\text{p}}\}\}$ implies accepting $\{ l_{\text{p}} \}$ is preferred over accepting $\{\}$, while $\sigma_{preferred}(AF') = \{\{\}\}$ implies accepting $\{\}$ is preferred over accepting $\{ l_{\text{p}} \}$.}
\label{fig:example2}
\end{figure}
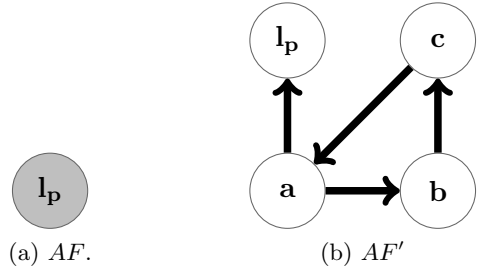

Let us now formally put argumentation-based decision functions in the context of economic rationality.
By considering Definition~\ref{clear-pref}, we define the \emph{clear preferences} principle for argumentation-based decision-making.
\begin{definition}[\emph{Clear Preferences} for Argumentation-based Decision-Making]
   Let $g \circ \sigma$ be an agent's argumentation-based decision function.
 Given any argumentation framework $AF = (AR, Attacks)$, it follows from $A^{*} = g \circ \sigma(AF)$ that the agent has established the following preference order: $\forall A \in 2^{AR}, A^{*} \succeq A$.
\end{definition}
In words, the agent's decision $A^{*}$ from $2^{AR}$ means that the agent prefers $A^{*}$ over any of the other sets in $2^{AR}$.

The \emph{clear preferences} principle is obvious when evaluating one-off argumentation-based decision-making.
However, when considering an argumentation process or \emph{dialogue}, during which arguments and attacks are added to an argumentation framework over time, it is clear that we need to consider the reference independence property as introduced in Definition~\ref{ref-in}.
We prove that, analogous to the economic principles, reference independence is implied by the \emph{clear preferences} property in the context of argumentation-based decision-making.
\begin{proposition}[\emph{Reference Independence} for Argumentation-based Decision-Making]\label{ref-dep-prop}
    \hspace{10pt} \\
    Let $AF = (AR, Attacks)$ and $AF' = (AR', Attacks')$ be two argumentation frameworks for which it holds true that $AR \subseteq AR'$ and let $g \circ \sigma$ be an argumentation-based decision function.
     An economically rational argumentation-based decision $A^{*} = g \circ \sigma(AF)$ implies that if $g \circ \sigma(AF') \subseteq AR$, then $g \circ \sigma(AF') = g \circ \sigma(AF)$.
\end{proposition}
\begin{proof}
    We provide a proof by contradiction.
    Let us suppose that $g \circ \sigma(AF) \neq g \circ \sigma(AF')$. It follows that $g \circ \sigma(AF) \not \subseteq g \circ \sigma(AF') \lor g \circ \sigma(AF') \not \subseteq g \circ \sigma(AF)$.
    \begin{enumerate}
        \item If $g \circ \sigma(AF) \not \subseteq g \circ \sigma(AF')$, then $\exists a \in g \circ \sigma(AF)$, such that $a \not \in g \circ \sigma(AF')$.
        \\ It follows that: \\ \emph{i)} By definition of $g \circ \sigma(AF)$, $\exists A^{*} \in 2^{AR}$, such that $\forall A \in 2^{AR}, A^{*} \succeq A \land a \in A^{*}$,
        \\ \emph{ii)} By definition of  $g \circ \sigma(AF')$, $\exists A'^{*} \in 2^{AR}$ such that $\forall A \in 2^{AR}, A'^{*} \succeq A \land a \not \in A'^{*}$.
        \\ $\implies$ Because $A'^{*} \neq A^{*}$, \emph{i}) contradicts \emph{ii}).
        \item If $g \circ \sigma(AF') \not \subseteq g \circ \sigma(AF)$, then $\exists a \in g \circ \sigma(AF')$ such that $a \not \in g \circ \sigma(AF)$.
       \\ It follows that: \\  \emph{i)} By definition of $g \circ \sigma(AF)$, $\exists A^{*} \in 2^{AR}$, such that $\forall A \in 2^{AR}, A^{*} \succeq A \land a \not \in A^{*}$,
        \\  \emph{ii)} By definition of $g \circ \sigma(AF')$, $\exists A'^{*} \in 2^{AR}$, such that $\forall A \in 2^{AR}, A'^{*} \succeq A \land a \in A'^{*}$.
        \\ $\implies$ Because $A'^{*} \neq A^{*}$, \emph{i}) contradicts \emph{ii}).
    \end{enumerate}
\end{proof}
In words, given a decision $A^{*}$ from $2^{AR}$ that implies $\forall A \in 2^{AR}, A^{*} \succeq A$, no decision $A'$ from $2^{AR'}$ with $AR' \supseteq AR$ should change the preferences implied by $A^{*}$.
Colloquially speaking, the proof is the \textquote{select a subset of a set}-equivalent to the economics proof that clear preferences imply consistent preferences for \textquote{select an item from a set}-scenarios, as for example presented by Rubinstein~\cite[p. 11]{rubinstein1998modeling}.
Note that in the proof, attack relations are irrelevant.
However, it can be assumed that given two argumentation frameworks $AF = (AR, Attacks)$ and $AF' = (AR', Attacks')$, with $AF'$ being an expansion, but not a normal expansion of $AF$, the reference independence property does not need to hold, because the knowledge that is modeled in the \textquote{original} argumentation framework $AF$ has changed. In contrast, if $AF \preceq_N AF'$, the relationships between arguments in $AR$ have not changed, which means that a rational decision-maker should maintain the previously established preference order on the items in the powerset of $AR$ given that no \textquote{new} arguments are considered valid conclusions.

\section{An Argumentation Principle for the Rational Man}
\label{paradigms}
From the \emph{rational man} properties that we have established in the context of argumentation-based decision functions, we derive a principle that an argumentation semantics $\sigma$ needs to satisfy to guarantee rational argumentation-based decision-making, given a decision function $g \circ \sigma$.
Ideally, the argumentation semantics $\sigma$ of an argumentation-based choice function $g \circ \sigma$ is uniquely defined ($|\sigma(AF)| = 1$ for any argumentation framework $AF$) and hence satisfies the clear preferences principle of the rational economic man, and  the aggregation function $g$ merely needs to flatten the set $\sigma$ returns.
An overview of semantics that fulfill this principle is provided in Appendix C, based on an analysis by Baumann~\cite{baumann2017nature}.
Let us note that the \emph{clear preferences} principle only implies clear preferences in one-off decision-making \emph{i.e.}, it is self-evident that \emph{i)} argumentation semantics that always return exactly one extension do not necessarily return extensions that imply consistent preference relations when resolving a sequence of normally expanding argumentation frameworks and \emph{ii)} argumentation semantics that may return multiple extensions may allow to \textquote{pick} an extension that implies consistent preference relations when resolving a normally expanding sequence.
The focus of this paper is on exactly this problem, which we analyze using the \emph{reference independence} principle.
\begin{definition}[\emph{Reference Independence} in Argumentation-based Decision Functions] 
    An argumentation-based decision function $g \circ \sigma$ satisfies the \emph{reference independence} principle iff for every two argumentation frameworks $AF = (AR, Attacks)$ and $AF' = (AR', Attacks')$, such that $AF \preceq_N AF'$, the following implication holds true, given an argumentation-based decision function $g \circ \sigma$:
    \begin{align*}
        \text{If } g \circ \sigma(AF') \subseteq AR \text{ then } g \circ \sigma(AF') = g \circ \sigma(AF)
    \end{align*}
\end{definition}
\noindent To illustrate the principle, let us introduce an example.
\begin{example}
    We have the following argumentation frameworks:
    \begin{enumerate}
        \item $AF = (AR, Attacks) =  (\{a, b\}, \{(a, b)\})$
        \item $AF' =  (AR', Attacks') =  (\{a, b, c\}, \{(a, b), (b, c), (c, a)\})$
    \end{enumerate}
    Note that $AF \preceq_N AF'$.
    Now, we apply complete semantics $\sigma^{complete}$ to both frameworks.
    \begin{enumerate}
        \item $g^{\cap} \circ \sigma_{complete}(AF') = \{a\}$, which implies $\forall Args \in 2^{\{a, b\}}, \{a\} \succeq Args$.
        \item $g^{\cap} \circ \sigma_{complete}(AF') = \{ \}$, which implies $\forall Args \in 2^{\{a, b, c\}}, \{\} \succeq Args$.
    \end{enumerate}
    $g^{\cap} \circ \sigma_{complete}(AF') \subset AR$ and $g^{\cap} \circ \sigma_{complete}(AF') \neq g^{\cap} \circ \sigma_{complete}(AF)$; \emph{i.e.},
    the decisions of $\{a\}$ given $\{a, b\}$ and $\{\}$ given $\{a, b, c\}$ are economically not rational, because the preference orders they imply are inconsistent: $g^{\cap} \circ \sigma_{complete}(AF)$ implies $\{a\} \succeq \{\}$, whereas $g^{\cap} \circ \sigma_{complete}(AF')$ implies $\{\} \succeq \{a\}$.
    Hence, it is clear that the argumentation semantics do not satisfy the \emph{reference independence} principle.
\end{example}
Let us again highlight that the principle applies by definition only to normal expansions of a framework ($AF \preceq_N AF'$).
We introduced this restriction because any expansion $AF'$ of $AF = (AR, Attacks)$ with $AF \not \preceq_N AF'$ that adds attacks between arguments of $AF$ is revising the assumptions about $AF$, \emph{i.e.} the \emph{ceteris paribus} assumption of economic rationality does not hold true.

Let us now introduce the reference independence principle for argumentation semantics.
\begin{definition}[Reference Independence Principle for Argumentation Semantics] 
    Let $\sigma$ be an argumentation semantics.
    $\sigma$ satisfies:
    \begin{itemize}
        \item \textbf{Strong reference independence}: iff $\sigma$ is universally defined and for every two argumentation frameworks $AF = (AR, Attacks)$ and $AF' = (AR', Attacks')$, such that $AF \preceq_N AF'$, $\forall E \in \sigma(AF), \forall E' \in \sigma(AF'), E' \not \subseteq AR \lor E' = E$.
        \item \textbf{Weak reference independence}: iff for every two argumentation frameworks $AF = (AR, Attacks)$ and $AF' = (AR', Attacks')$, such that $AF \preceq_N AF'$, $\forall E \in \sigma(AF)$, it holds true that $\exists E' \in \sigma(AF')$, such that $E' \not \subseteq AR \lor E' = E$.
    \end{itemize}
\end{definition}
It is clear that strong reference independence cannot be guaranteed by a \textquote{reasonable} abstract argumentation semantics.
To illustrate this, let us claim that $\sigma((\{a\}, \{\})) = \{\{a\}\}$, $\sigma((\{b\}, \{\})) = \{\{b\}\}$ are \textquote{reasonable} assumptions about the behavior of an argumentation semantics.
Let us now consider the frameworks $AF = (AR, Attacks)$, such that $AF = (\{a, b\}, \{(a, b), (b, a)\})$ and $AF' = (AR', Attacks')$, such that $AF' = (\{a, b, c\}, \{(a, b), (b, a), \\ (b, c), (c, b), (c, a), (a, c)\})$ (note that $AF \preceq_N AF'$). 
One can see that $\sigma(AF)$ needs to return $\{\{a\}, \{b\}\}$ (if it returned the empty set, we could not achieve weak reference independence when expanding from an argumentation framework containing only $a$ ($(\{a\}, \{\})$) or $(\{b\}, \{\})$), whereas $\sigma(AF')$ needs to return $\{\{a\}, \{b\}, \{c\}\}$ to be weakly reference independent. However $\sigma(AF')$ is not strongly reference independent, because $\forall E \in \sigma(AF)$ it does not hold true that $\forall E' \in \sigma(AF'), (E' \not \subseteq AR \lor E' = E)$:
\begin{itemize}
    \item Let $E$ be $\{b\}$ and let $E'$ be $\{a\}$. Note that $\{b\} \in \sigma(AF)$ and $\{a\} \in \sigma(AF')$. $\{a\} \subseteq AR \land \{a\} \neq \{b\}$.
    \item Let $E$ be $\{a\}$ and let $E'$ be $\{b\}$. Note that $\{a\} \in \sigma(AF)$ and $\{b\} \in \sigma(AF')$. $\{b\} \subseteq AR \land \{b\} \neq \{a\}$.
\end{itemize}
In contrast, we show further below that weak reference independence can be satisfied by an argumentation semantics; hence it is the more relevant property in the context of abstract argumentation and we may refer to it as \emph{reference independence} without any further qualifier.
It is important to highlight that the strong reference independence principle is still a crucial building block for this paper, as it can be derived in a more obvious manner from the reference independence property of economic rationality; showing that relaxing this principle is necessary to motivate the weak reference independence principle.
Let us colloquially summarize how we have transformed a formal model of an economically rational decision-maker (\textquote{rational economic man}) into the abstract argumentation principle of weak reference independence.
\begin{enumerate}
    \item Rational economic man selects a subset of a set of items, which are, in the case of abstract argumentation, arguments in an argumentation framework.
    \item The selection implies a preference order (partial order) on the powerset of items/arguments: the selected set is preferred over all other items in the powerset.
    \item When adding items to the set, the preference order on the powerset of items remains consistent as long as the properties of the initial items do not change.
    Consequently, in an argumentation scenario, when expanding an argumentation framework, the preference order on the powerset of arguments remains consistent as long as the relations among the initial arguments remain unchanged, \emph{i.e.}, as long as the expansion is \emph{normal}. Let us highlight that attacks from \textquote{new} to \textquote{initial} arguments must not affect the preference order as long as no \textquote{new} argument is considered a valid conclusion and if a \textquote{new} argument is considered a valid conclusion (is part of the extension) the new preference order is by definition consistent with the initial preference order.
\end{enumerate}

\subsection{Reference Independence and Monotony}
Considering that abstract argumentation is a method for non-monotonic reasoning, and given that the \emph{reference independence} property may seem -- at first glance -- to imply monotony, let us provide an intuition for distinguishing between reference independence and the different forms of monotony as defined in Section~\ref{inference} in the context of abstract argumentation.
In particular, we provide proofs that show the following, given any argumentation semantics whose extensions are maximal conflict-free sets, and in case an implication relationship holds, also given any argumentation semantics whose extensions are maximal admissible sets:
\begin{itemize}
    \item Strong monotony implies strong reference independence and weak monotony implies weak reference independence, but not vice versa.
    \item Strong rational monotony does not imply strong reference independence and vice versa. Weak rational monotony implies weak reference independence, but not vice versa. We observe that in contrast to weak reference independence, weak rational monotony is not satisfied by any non-naive semantics that is evaluated in this paper; \emph{i.e.}, weak rational monotony is too strict to be useful in the context of abstract argumentation.
    \item Strong/weak cautious monotony does not imply strong/weak reference independence and vice versa.
\end{itemize}
By definition, strong monotony implies strong rational monotony and strong cautious monotony, and weak monotony implies weak rational monotony and weak cautious monotony.
We have observed that typically, strong/weak rational monotony implies strong/weak cautious monotony (see: Section~\ref{inference}).
Figure~\ref{fig:implications} visualizes the relations between monotony, cautious monotony, rational monotony, and reference independence.
\begin{figure}
        \centering
        \begin{tikzpicture}[
            noanode/.style={dashed, circle, draw=black!60, minimum size=10mm, font=\bfseries},
            unode/.style={font=\bfseries},
            anode/.style={circle, fill=lightgray, draw=black!60, minimum size=10mm, font=\bfseries},
            ]
           \node[unode]    (m)    at(0,2)  {Monotony};
            \node[unode]    (cm)    at(10,2)  {Cautious Monotony};
            \node[unode]    (rm)    at(5,2)  {Rational Monotony};
             \node[unode]    (ri)    at(7.5,0)  {Reference Independence};
             \path [->, line width=0.5mm]  (rm) edge node[left] {weak} (ri);
            \path [->, line width=1mm]  (m) edge node[left] {} (rm);
            \path [->, line width=1mm]  (m) edge node[left] {} (ri);
            \path [->, line width=1mm]  (rm) edge node[left] {} (cm);
        \end{tikzpicture}
\caption{Monotony implies reference independence, rational monotony and cautious monotony (given any semantics that satisfies maximal conflict-freeness or maximal admissibility). Rational monotony typically implies cautious monotony (see Section~\ref{inference}). Weak rational monotony implies weak reference independence; however, rational monotony is violated by all non-naive argumentation semantics that are evaluated in this paper.}
\label{fig:implications}
\end{figure}
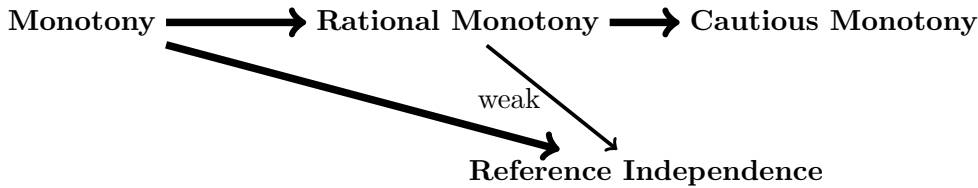
 We conclude that only monotony and (in the \emph{weak} case, rational monotony), which are too strict for common-sense reasoning approaches, imply reference independence. Hence, we can conclude that reference independence is novel and useful as a property for analyzing abstract argumentation approaches from the perspective of economic rationality.

Let us motivate our decision to require semantics that are based on maximal conflict-free sets or maximal admissible sets in our proofs instead of merely requiring conflict-freeness as the only principle that needs to be satisfied.
\begin{example}
 Let $AF = (\{a\}, \{\})$ and $AF' = (\{a, b\}, \{\})$ be argumentation frameworks and let $\sigma$ be an argumentation semantics such that $\sigma(AF) = \{\{\}\}$ and $\sigma(AF') = \{\{a\}\}$. While $\sigma$ returns conflict-free (and also admissible) sets for both argumentation frameworks, the behavior of $\sigma$ is clearly flawed from a common-sense perspective: the structure of the argumentation framework provides no reason to consider $a$ but not $b$ part of a valid conclusion. For semantics with such behavior, the implication relationship between monotony and reference independence that we are demonstrating below does not hold true.
\end{example}
Introducing a new property that is weaker than maximal conflict-freeness and maximal admissibility, yet does not allow for semantics with \textquote{unreasonable} behavior as shown in the example can be considered relevant future work, but is not in the scope of this paper.
The proofs of the observations regarding weak monotony properties and reference independence follow below. For the sake of conciseness, the corresponding proofs regarding \emph{strong} monotony and reference independence, which are analogous to the proofs regarding the corresponding \emph{weak} properties, are available in Appendix A.
In Appendix B we show that it holds true for maximal admissible set-based semantics that strong monotony implies strong reference independence and weak monotony implies weak reference independence.

Let us first prove that weak reference independence does not imply weak monotony.
\begin{proposition}[Weak Reference Independence does not Imply Weak Monotony]
    There exist an argumentation semantics $\sigma$ whose extensions are maximal conflict-free sets (w.r.t. set inclusion) and argumentation frameworks $AF = (AR, Attacks)$, $AF' = (AR', \\ Attacks')$, such that $AF \preceq_N AF'$, such that the following statement does not hold true:
    \begin{align*}
        {}&\text{If }(\forall E \in \sigma(AF), \exists E' \in \sigma(AF'), \text{ such that } (E' \not \subseteq AR \lor E' = E)) \\
        {}& \phantom{eee} \text{ then } \\
        {}& (\forall E \in \sigma(AF), \exists E' \in \sigma(AF'), \text{ such that }  E \subseteq E') 
    \end{align*}
\end{proposition}
\begin{proof}
    Let us provide a proof by counter-example.
    \begin{enumerate}
        \item Let $AF = (AR, Attacks) = (\{a\}, \{\})$, $AF' = (\{a, b\}, \{(b, a)\})$ and let us take CF2 semantics, denoted by $\sigma_{cf2}$, as an example of a semantics whose extensions are maximal conflict-free sets. Note that $AF \preceq_N AF'$.
        \item $\sigma_{cf2}(AF) = \{\{a\}\}$ and $\sigma_{cf2}(AF') = \{\{b\}\}$. Let $E = \{a\}, E \in \sigma_{cf2}(AF)$. $\forall E' \in \sigma(AF')$ it holds true that $E \not \subseteq E'$. The condition for weak monotony does not hold.
        \item $\exists E' \in \sigma(AF')$ such that $E' \not \subseteq AR$. The condition for weak reference independence holds. We have proven the proposition.
    \end{enumerate}
\end{proof}
The proof below shows that weak monotony implies weak reference independence. The intuition is that given any maximal conflict-free set $S$ of an argumentation framework $AF = (AR, Attacks)$, adding arguments (and attacks) to the argumentation framework (without changing attacks between existing arguments) will never cause any argument in $AR \setminus S$ to be conflict-free w.r.t. $S$. This means it is not possible to \textquote{add} arguments within $AR$ to a maximal conflict-free set by normally expanding $AF$, which in turn means that to violate reference independence we must in our expansion of $AF$ successfully attack arguments that are accepted in $AF$, which then violates monotony.
\begin{proposition}[Weak Monotony implies Weak Reference Independence]\label{monotony-implies-ref-dep}
    For every argumentation semantics $\sigma$ whose extensions are maximal conflict-free sets (w.r.t. set inclusion) and every two argumentation frameworks $AF = (AR, Attacks), AF' = (AR', \\ Attacks')$, such that $AF \preceq_N AF'$, the following statement holds true:
    \begin{align*}
       {}& \text{If }  (\forall E \in \sigma(AF), \exists E' \in \sigma(AF'), \text{ such that } E \subseteq E') \\
       {}& \phantom{eee} \text{ then } \\
       {}& (\forall E \in \sigma(AF), \exists E' \in \sigma(AF'), \text{ such that } E' \not \subseteq AR \lor E' = E)
    \end{align*}
\end{proposition}
\begin{proof}
    \phantom{eee}
    \begin{enumerate}
        \item Given $\sigma$ such that each all $\sigma$-extensions are maximal conflict-free sets (w.r.t. set inclusion), $\forall E \in \sigma(AF), E' \in \sigma(AF')$, $E$ and $E'$ are maximal conflict-free (w.r.t. set inclusion). It follows that, if $E \subseteq E'$, we have the following cases:
        \\ \emph{i)} E = E' or
        \\ \emph{ii)} $E \subset E'$, which implies that $\exists a \in E'$, such that $a \not \in E, a \in AR' \setminus AR$, from which it follows that $E' \not \subseteq AR$.
        \item Consequently, by i) and ii) the following statement holds true:
        \begin{align*}
            \forall E \in \sigma(AF), \forall E' \in \sigma(AF'), \text{ if } E \subseteq E' \text{ then } E' \not \subseteq AR \lor E' = E
        \end{align*}
        Hence, the proposition (which is implied by the proven statement) holds true.
    \end{enumerate}
\end{proof}
We can prove that weak \emph{cautious} monotony does not imply weak reference independence.
\begin{proposition}[Weak Cautious Monotony does not Imply Weak Reference Independence]
    There exists an argumentation semantics $\sigma$ whose extensions are maximal conflict-free sets (w.r.t. set inclusion) and argumentation frameworks $AF = (AR, Attacks), AF' = (AR', Attacks')$, such that $AF \preceq_N AF'$ and the following statement does not hold true:
    \begin{align*}
        &{} \forall E \in \sigma(AF), \\
        &{}\quad \text{If }  (\{(a,b) \mid (a,b) \in Attacks', a \in AR' \setminus AR, b \in E \}=\emptyset \\
         &{} \quad \text{ implies } \exists E' \in \sigma(AF') \text{ such that } E \subseteq E') \\
        &{} \quad \quad \text{ then } \\
        &{} \quad (\exists E' \in \sigma(AF') \text{ such that } E' \not \subseteq AR \lor E' = E))
    \end{align*}
\end{proposition}
\begin{proof}
    Let us provide a proof by counter-example.
    \begin{enumerate}
        \item Let $AF = (AR, Attacks) = (\{c, d\}, \{(c, d), (d, c)\})$, $AF' = (\{c, d, e\}, \{(c, d), (d, c), (d, e), \\ (e, c), (e, e)\})$ and let us take stage semantics, denoted by $\sigma_{stage}$, as an example of a semantics whose extensions are maximal conflict-free sets. Note that $AF \preceq_N AF'$.
        \item $\sigma_{stage}(AF) = \{\{c\}, \{d\}\}$ and  $\sigma_{stage}(AF') = \{\{d\}\}$.
        \item Consequently, given $E = \{c\}$, $\{(a,b) \mid (a,b) \in Attacks', a \in AR' \setminus AR, b \in E \}=\emptyset$ is false and it holds true that $\forall E \in \sigma(AF)$, $(\{(a,b) \mid (a,b) \in Attacks', a \in AR' \setminus AR, b \in E \}=\emptyset$ implies $\exists E' \in \sigma(AF') \text{ such that } E \subseteq E')$.
        \item However, it does not hold true that $\forall E \in \sigma(AF), \exists E' \in \sigma(AF')$, such that $(E' \not \subseteq AR \lor E' = E)$. We have proven the proposition.
    \end{enumerate}
\end{proof}
In a similar manner, we can prove that weak reference independence does not imply weak cautious monotony.
\begin{proposition}[Weak Reference Independence does not Imply Weak Cautious Monotony]\label{ref-dep-not-monotony}
    There exists an argumentation semantics $\sigma$ whose extensions are maximal conflict-free sets (w.r.t. set inclusion) and argumentation frameworks $AF = (AR, Attacks), AF' = (AR', Attacks')$ such that $AF \preceq_N AF'$ and the following statement does not hold true:
    \begin{align*}
         &{} \forall E \in \sigma(AF), \\
         &{} \quad (\text{if } (\exists E' \in \sigma(AF'), \text{ such that } E' \not \subseteq AR \lor E' = E) \\
         &{} \quad \quad \text{ then } \\
         &{} \quad (\{(a,b) \mid (a,b) \in Attacks', a \in AR' \setminus AR, b \in E \}=\emptyset \text{ implies } \\
         &{} \quad \quad \exists E' \in \sigma(AF'), \text{ such that } E \subseteq E')
    \end{align*}
\end{proposition}
\begin{proof}
    We provide the following proof by counter-example.
    \begin{enumerate}
        \item Let $AF = (AR, Attacks) = (\{c, d, e\}, \{(c, d), (d, e), (e, c)\})$, \\ $AF' = (\{c, d, e, f\}, \{(c, d), (d, e), (e, c), (f, d)\})$ and let us take stage semantics, denoted by $\sigma_{stage}$, as an example of a semantics whose extensions are maximal conflict-free sets. Note that $AF \preceq_N AF'$.
        \item $\sigma_{stage}(AF) = \{\{c\}, \{d\}, \{e\}\}$ and  $\sigma_{stage}(AF') = \{\{e, f\}\}$.
        \item Given the extension $E \in \sigma_{stage}(AF)$, $E = \{c\}$, it holds true that $\{(a,b) \mid (a,b) \in Attacks', a \in AR' \setminus AR, b \in E \}=\emptyset$.
        \item Given $E' = \{e, f\}$ as the only extension in $\sigma(AF')$ it holds true that $E' \not \subseteq AR$; the result does not violate the weak reference independence property.
        \item However, given $E' = \{e, f\}$ as the only extension in $\sigma(AF')$ and given $E = \{c\}, E \in \sigma_{stage}(AF)$, we have $E \not \subseteq E'$, which violates weak cautious monotony. Hence, we have proven the proposition.
    \end{enumerate}
\end{proof}
Figure~\ref{fig:examples-rat-mon-1} depicts the frameworks used in the proof of Proposition~\ref{ref-dep-not-monotony}.
\begin{figure}
    \subfloat[$AF$.]{
        \begin{tikzpicture}[
            noanode/.style={dashed, circle, draw=black!60, minimum size=10mm, font=\bfseries},
            unanode/.style={circle, draw=black!60, minimum size=10mm, font=\bfseries},
            anode/.style={circle, fill=lightgray, draw=black!60, minimum size=10mm, font=\bfseries},
            ]
             \node[unanode]    (c)    at(0,2)  {c};
            \node[unanode]    (d)    at(2,2)  {d};
            \node[unanode]    (e)    at(0,4)  {e};
            \path [->, line width=1mm]  (c) edge node[left] {} (d);
            \path [->, line width=1mm]  (d) edge node[left] {} (e);
            \path [->, line width=1mm]  (e) edge node[left] {} (c);
        \end{tikzpicture}
    }
    \hspace{50pt}
    \centering
    \subfloat[$AF'$]{
        \begin{tikzpicture}[
            noanode/.style={dashed, circle, draw=black!60, minimum size=10mm, font=\bfseries},
            anode/.style={circle, fill=lightgray, draw=black!60, minimum size=10mm, font=\bfseries},
            ]
             \node[noanode]    (c)    at(0,2)  {c};
            \node[noanode]    (d)    at(2,2)  {d};
            \node[anode]    (e)    at(0,4)  {e};
            \node[anode]    (f)    at(2,4)  {f};
            \path [->, line width=1mm]  (c) edge node[left] {} (d);
            \path [->, line width=1mm]  (d) edge node[left] {} (e);
            \path [->, line width=1mm]  (e) edge node[left] {} (c);
            \path [->, line width=1mm]  (f) edge node[left] {} (d);
        \end{tikzpicture}
    }
\caption{Example: weak reference independence does not imply weak cautious monotony.} 
\label{fig:examples-rat-mon-1}
\end{figure}
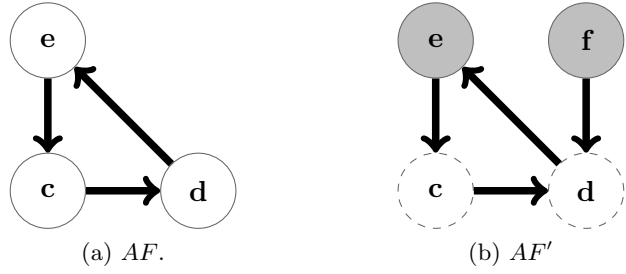
Let us now prove that weak rational monotony implies weak reference independence.
\begin{proposition}[Weak Rational Monotony  Implies Weak Reference Independence]\label{rational-monotony-not-ref-dep}
    For every argumentation semantics $\sigma$ whose extensions are maximal conflict-free sets (w.r.t. set inclusion) and every two argumentation frameworks $AF = (AR, Attacks), AF' = (AR', \\ Attacks')$, such that $AF \preceq_N AF'$, the following statement holds true:
    \begin{align*}
             &{} \forall E \in \sigma(AF), \\
             &{} \quad (\text{If } (\{(a,b) \mid (a,b) \in Attacks', a \in UE'_{new}, b \in E \}=\emptyset \\
              &{} \quad \quad \text{ implies } \exists E' \in \sigma(AF'), \text{ such that } E \subseteq E') \\ 
             &{} \quad \quad \text{ then } \\
             &{}\quad (\exists E' \in \sigma(AF'), \text{ such that } (E' \not \subseteq AR \lor E' = E))),
        \end{align*}
        where $UE'_{new} = \bigcup_{E' \in \sigma(AF')} (E' \setminus AR)$.
\end{proposition}
\begin{proof}
    Given any $E \in \sigma(AF)$, we have two cases for which the weak rational monotony condition is satisfied.
    \begin{description}
        \item[Case 1:] $\{(a,b) \mid (a,b) \in Attacks', a \in UE'_{new}, b \in E \} \neq \emptyset$: \\
        If $\{(a,b) \mid (a,b) \in Attacks', a \in UE'_{new}, b \in E \} \neq \emptyset$, then $\exists E' \in \sigma(AF'), \text{ such that } (E' \not \subseteq AR)$. Consequently, weak reference independence is satisfied.
         \item[Case 2:] $\{(a,b) \mid (a,b) \in Attacks', a \in UE'_{new}, b \in E \} = \emptyset$ and $\exists E' \in \sigma(AF'), \text{ such that } E \subseteq E'$: \\
         Because it holds true that $\exists E' \in \sigma(AF'), \text{ such that } E \subseteq E'$ and because all $\sigma$-extensions are maximal conflict-free sets (w.r.t. set inclusion), from the proof of Proposition~\ref{monotony-implies-ref-dep} it follows that weak reference independence is satisfied. We have proven the proposition.
    \end{description}
\end{proof}
The proof that weak rational monotony implies weak reference independence in the case of maximal admissible set-based semantics is provided in the Appendix.
However, let us informally observe that weak rational monotony is violated by all semantics that are analyzed in this paper, with the exception of naive semantics. For this, we merely need to consider the following examples:
\begin{itemize}
    \item Complete, grounded, preferred, stable, ideal, semi-stable, eager: $AF_1 = (\{a\}, \{\})$, $AF'_1 = (\{a, b, c\}, \{(a, b), (b, c), (c, a)\})$, $AF_1 \preceq_N AF'_1$\footnote{This example can also be applied to the weakly admissible set-based semantics as introduced by Baumann \emph{et al.}~\cite{BaumannBU20} to observe that these semantics violate weak rational monotony.};
    \item Stage, CF2, Stage2: $AF_2 = (\{a, b, c\}, \{(a, b), \\ (b, c), (c, a)\})$, $AF'_2 = (\{a, b, c, d\}, \{(a, b), (b, c), (c, a), (d, c)\})$, $AF_2 \preceq_N AF'_2$.
\end{itemize}

Now, let us prove that weak reference independence does not imply weak rational monotony.
\begin{proposition}[Weak Reference Independence does not Imply Weak Rational Monotony]
    There exists an argumentation semantics $\sigma$ whose extensions are maximal conflict-free sets and argumentation frameworks $AF = (AR, Attacks), AF' = (AR', Attacks')$, such that $AF \preceq_N AF'$, such that the following statement does not hold true:
    \begin{align*}
             &{} \forall E \in \sigma(AF), \\
             &{} \quad (\text{if } (\exists E' \in \sigma(AF'), \text{such that } E' \not \subseteq AR \lor E' = E)) \\
             &{} \quad \quad  \text{then} \\
             &{} \quad (\{(a,b) \mid (a,b) \in Attacks', a \in UE'_{new}, b \in E \}=\emptyset \\
             &{} \quad \quad \text{ implies }  \exists E' \in \sigma(AF'), \text{ such that } E \subseteq E')),
    \end{align*}
    where $UE'_{new} = \bigcup_{E' \in \sigma(AF')} (E' \setminus AR)$
\end{proposition}
\begin{proof}
    Let us provide a proof by counter-example.
    \begin{enumerate}
        \item Let $AF = (AR, Attacks) = (\{c, d, e\}, \{(c, d), (d, e), (e, c)\})$, \\ $AF' = (\{c, d, e, f\}, \{(c, d), (d, e), (e, c), (f, d)\})$ and let us take stage semantics, denoted by $\sigma_{stage}$, as an example of a semantics whose extensions are maximal conflict-free sets. Note that $AF \preceq_N AF'$.
        \item $\sigma_{stage}(AF) = \{\{c\}, \{d\}, \{e\}\}$ and  $\sigma_{stage}(AF') = \{\{e, f\}\}$.
        \item Given $E' = \{e, f\}$ as the only extension in $\sigma(AF')$ it holds true that $E' \not \subseteq AR$; the result does not violate the weak reference independence property.
        \item However, given $E' = \{e, f\}$ as the only extension in $\sigma(AF')$ and given $E = \{c\}, E \in \sigma_{stage}(AF)$, it holds true that $\{(a,b) \mid (a,b) \in Attacks', a \in UE'_{new}, b \in E \}=\emptyset$ and $E \not \subseteq E'$, which violates weak rational monotony. Hence, we have proven the proposition.
    \end{enumerate}
\end{proof}

\subsection{Weak Reference Independence, Directionality, and SCC-Recursiveness}
\label{directionality}
Two argumentation principles that have received much attention because, colloquially speaking, they reflect the intuitive idea to traverse an argumentation graph in a \textquote{top-down} manner, are directionality and SCC-recursiveness~\cite{BARONI2007675}.
Let us highlight (and show by formal analysis) that weak reference independence does not imply directionality and SCC-recursiveness (and vice versa).
Let us first provide the definition of unattacked sets as a preliminary.
\begin{definition}[Unattacked Sets~\cite{BARONI2007675}]~\label{unattacked-sets}
Let $AF = (AR, Attacks)$ be an argumentation framework.
A set $S \subseteq AR$ is \emph{unattacked} iff $\nexists a \in AR \setminus S$ such that $a$ attacks $S$.  $US(AF)$ denotes all unattacked sets in $AF$.
\end{definition}
Now, let us provide the definition of the directionality principle
\begin{definition}[Directionality~\cite{BARONI2007675}]\label{PropertyDirectionality}
An argumentation semantics $\sigma$ is directional iff for every argumentation framework $AF = (AR, Attacks)$, for every unattacked set of arguments $U \subseteq AR$ it holds true that \\ $\sigma(AF \downarrow_{U}) = \{E \cap U | E \in \sigma(AF)\}$.
\end{definition}
Let us now introduce the SCC-recursiveness principle, on which semantics like CF2 and stage2 are based.
\begin{definition}[SCC-recursiveness~\cite{BARONI2005162}]
    Let $\sigma$ be an argumentation semantics. $\sigma$ is SCC-recursive iff $\sigma(AF) = GF(AF, Attacks)$, where for every $AF = (AR, Attacks)$, $C \subseteq AF$, $GF(AF, C) \subseteq 2^{AR}$ is defined as follows.
    For every $E \subseteq AR, E \in GF(AF, C)$ iff
    \begin{itemize}
        \item if $|SCSS(AF)| = 1, E \in BF(AR, C)$;
        \item otherwise, $\forall S \in SCSS(AF)$, $(E \cap S) \in GF(AF \downarrow_{UP_{AF}(S, E)}, U_{AF}(S, E) \cap C)$, 
    \end{itemize}
    where $BF(AF, C)$ is a function that, given an argumentation framework $AF = (AR, Attacks)$, such that $|SCSS(AF)| = 1$ and a set $C \subseteq AR$, returns a subset of $2^{A}$. 
\end{definition}
To show that weak reference independence does not imply directionality and SCC-recursiveness, let us prove that naive semantics, which does not satisfy directionality and SCC-recursiveness~\cite{van2017principle}, satisfies the weak reference independence principle.
\begin{proposition}~\label{naive-ref-dep}
    Let $\sigma_{naive}$ be naive argumentation semantics. For every $AF = (AR, Attacks), AF' = (AR', Attacks')$, such that $AF \preceq_N AF', \forall E \in \sigma_{naive}(AF)$, the following statement holds true:
    \begin{align*}
        \exists E' \in \sigma_{naive}(AF'), \text{ such that } E' \not \subseteq AR \lor E' = E 
    \end{align*}
\end{proposition}
\begin{proof}
    $\forall E \in \sigma_{naive}(AF)$, $E$ is a maximal conflict-free set w.r.t. set inclusion. $\sigma_{naive}(AF')$ contains all $S \subseteq AR'$, such that $S$ is a maximal conflict-free set w.r.t. set inclusion. It follows that because $AF \preceq AF'$, we have two cases $\forall E \in \sigma_{naive}(AF)$:
    \begin{enumerate}
        \item $\exists a \in AR' \setminus AR$, such that $a$ does not attack $E$, $E$ does not attack $a$, and $a$ does not attack $a$. It follows that $E \cup \{a\}$ is a conflict-free set and hence it holds true that $\exists E' \in \sigma_{naive}(AF)$, such that $E' \not \subseteq AR$.
        \item $\nexists a \in AR' \setminus AR$, such that $a$ does not attack $E$, $E$ does not attack $a$, and $a$ does not attack $a$. It follows that $E$ is a maximal conflict-free (w.r.t. set inclusion) subset of $AR'$ and hence $\exists E' \in \sigma_{naive}(AF)$, such that $E' = E$.
    \end{enumerate}
    We have proven the proposition.
\end{proof}

From our analysis in the next section, it follows that directionality and SCC-recursiveness do not imply weak reference independence: the directional and SCC-recursive stage2 semantics does not satisfy weak reference independence.
However, the directional and SCC-recursive CF2 semantics satisfies weak reference independence, \emph{i.e.}, SCC-recursiveness and directionality do not imply violation of weak directionality.
\section{Reference Dependence in Abstract Argumentation}
\label{rational-semantics}
Let us show violation of the \emph{weak reference independence} principle of complete, stable preferred, semi-stable, as well as grounded, ideal and eager semantics (credulous and skeptical).
These observations are formalized by the following proposition.
Note that for the sake of providing more concise proposition and proofs, we denote credulous semantics (typically denoted by $\sigma$) by $\sigma^{i}$, with $i$ being the identity function, \emph{i.e.},  $\sigma^{i}(AF) =\sigma(AF)$ for every argumentation framework AF.
\begin{proposition}\label{irrational-1}
    Let $\sigma^{y}_{x}$ be an argumentation semantics, such that $x \in \{complete, stable, grounded, preferred, \\ ideal, \text{semi-stable}, eager\}$ and 
    $y \in \{\cap, i\}$. There exist argumentation frameworks $AF = (AR, Attacks)$ and $AF' = (AR', Attacks')$, such that $AF \preceq_N AF'$, for which the following statement does not hold true:
    \begin{align*}
        \forall E \in \sigma^{y}_{x}(AF), (\exists E' \in \sigma^{y}_{x}(AF'), \text{such that } (E' \not \subseteq AR \lor E' = E))
    \end{align*}
\end{proposition}
\begin{proof}\label{ref-dep-p1a}
    We provide a proof by counter-example. Let us introduce the following argumentation frameworks:
    \begin{itemize}
        \item $AF = (AR, Attacks) = (\{a, b\}, \{(a, b)\})$;
        \item $AF' = (AR', Attacks') = (\{a, b, c\}, \{(a, b), (b, c), (c, a)\})$.
    \end{itemize}
    We can see that $AF \preceq_N AF'$. 
    The argumentation frameworks are resolved as follows:
    \begin{itemize}
        \item $\sigma^{y}_{x}(AF) = \{\{a\}\}$.
        \item If $x = stable$, $\sigma^{y}_{x}(AF') = \{\}$; otherwise, $\sigma^{y}_{x}(AF') = \{\{\}\}$. Note that $\{\}\ \subseteq AR \land \{\}\ \neq \{a\}$. We have proven the proposition.
    \end{itemize}
\end{proof}
The proof that skeptical stage, stage2, and CF2 semantics do not satisfy the weak reference independence principle can be provided in the same way.
\begin{proposition}\label{irrational-2}
    Let $\sigma^{\cap}_{x}$ be an argumentation semantics, such that $x \in \{ stage, stage2, CF2 \}$. There exist argumentation frameworks $AF = (AR, Attacks)$ and $AF' = (AR', Attacks')$, such that $AF \preceq_N AF'$, for which the following statement does not hold true:
    \begin{align*}
        \forall E \in \sigma^{\cap}_{x}(AF), (\exists E' \in \sigma^{\cap}_{x}(AF'),  \text{such that } (E' \not \subseteq AR \lor E' = E))
    \end{align*}
\end{proposition}
\begin{proof}\label{ref-dep-p1b}
    To provide a proof by counter-example, let us again introduce the following argumentation frameworks:
    \begin{itemize}
        \item $AF = (AR, Attacks) = (\{a, b\}, \{(a, b)\})$;
        \item $AF' = (AR', Attacks') = (\{a, b, c\}, \{(a, b), (b, c), (c, a)\})$.
    \end{itemize}
    We can see that $AF \preceq_N AF'$. 
    The argumentation frameworks are resolved as follows:
    \begin{itemize}
        \item $\sigma^{\cap}_{x}(AF) = \{\{a\}\}$.
        \item $\sigma^{\cap}_{x}(AF') = \{\{\}\}$.
    \end{itemize}
    Note that $\{\}\ \subseteq AR \land \{\}\ \neq \{a\}$. We have proven the proposition. 
\end{proof}
Note that the counter-example in the proof of Proposition~\ref{irrational-2} is analogous to the counter-example in the proof of Proposition~\ref{irrational-1}.
In words, given $AF$, $\{a\}$ is preferred over $\{\}$ and given $AF'$, $\{\}$ is preferred over $\{a\}$. Hence, the semantics violate the \emph{reference independence} principle; adding an element $c$ to the set of elements $\{a, b\}$ can affect the preference relation an agent has on elements in $2^{\{a, b\}}$.
For credulous $stage$ and $stage2$ semantics, another example proves violation of the weak reference independence principle. 
\begin{proposition}\label{irrational-3}
    Let $\sigma^{i}_{x}$ be an argumentation semantics, such that $x \in \{ stage, stage2 \}$. There exist argumentation frameworks $AF = (AR, Attacks)$ and $AF' = (AR', Attacks')$, such that $AF \preceq_N AF'$, for which the following statement does not hold true:
    \begin{align*}
        \forall E \in \sigma^{i}_{x}(AF), (\exists E' \in \sigma^{i}_{x}(AF'),  \text{such that } (E' \not \subseteq AR \lor E' = E))
    \end{align*}
\end{proposition}
\begin{proof}\label{ref-dep-p1c}
    To provide a proof by counter-example, let us again introduce the following argumentation frameworks:
    \begin{itemize}
        \item $AF = (AR, Attacks) = (\{a, b\}, \{(a, b), (b, a)\})$;
        \item $AF' = (AR', Attacks') = (\{a, b, c\}, \{(a, b), (b, a), (b, c),  (c, b),  (c, c)\})$.
    \end{itemize}
    We can see that $AF \preceq_N AF'$. 
    The argumentation frameworks are resolved as follows:
    \begin{itemize}
        \item $\sigma^{i}_{x}(AF) = \{\{a\}, \{b\}\}$.
        \item $\sigma^{i}_{x}(AF') = \{\{b\}\}$.
    \end{itemize}
    Note that $\{b\}\ \subseteq AR \land \{b\}\ \neq \{a\}$. We have proven the proposition. 
\end{proof}
Figure~\ref{fig:examples-rat-mon-3} shows the frameworks that the proof uses.
\begin{figure}
    \subfloat[$AF$.]{
        \begin{tikzpicture}[
            noanode/.style={dashed, circle, draw=black!60, minimum size=10mm, font=\bfseries},
            unanode/.style={circle, draw=black!60, minimum size=10mm, font=\bfseries},
            anode/.style={circle, fill=lightgray, draw=black!60, minimum size=10mm, font=\bfseries},
            ]
            \node[unanode]    (a)    at(0,4)  {a};
            \node[unanode]    (b)    at(0,2)  {b};
            \path [->, line width=1mm]  (a) edge node[left] {} (b);
            \path [->, line width=1mm]  (b) edge node[left] {} (a);
        \end{tikzpicture}
    }
    \hspace{50pt}
    \centering
    \subfloat[$AF'$]{
        \begin{tikzpicture}[
            noanode/.style={dashed, circle, draw=black!60, minimum size=10mm, font=\bfseries},
            anode/.style={circle, fill=lightgray, draw=black!60, minimum size=10mm, font=\bfseries},
            ]
            \node[noanode]    (a)    at(0,4)  {a};
            \node[anode]    (b)    at(0,2)  {b};
            \node[noanode]    (c)    at(2,2)  {c};
            \path [->, line width=1mm]  (a) edge node[left] {} (b);
            \path [->, line width=1mm]  (b) edge node[left] {} (a);
            \path [->, line width=1mm]  (b) edge node[left] {} (c);
             \path [->, line width=1mm]  (c) edge node[left] {} (b);
            \draw[->, line width=1mm] (c.+90) arc (180:180-264:4mm);
        \end{tikzpicture}
    }
\caption{Inconsistent preferences: Given the extensions $\sigma_{stage}(AF) = \{\{a\}, \{b\}\}$, an agent can \textquote{pick} the extension $\{a\}$, which then implies the preference $\{ a \} \preceq \{ b \}$, while $\sigma_{stage}(AF')$ implies the preference $\{ b \} \preceq \{ a \}$.}
\label{fig:examples-rat-mon-3}
\end{figure}
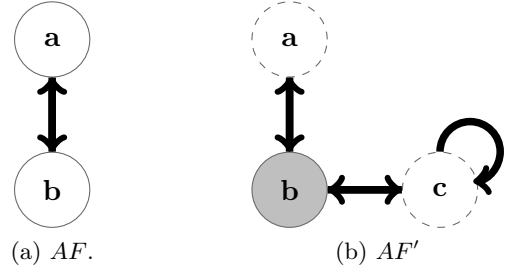
Finally, let us prove that CF2 semantics is weakly reference independent. Note that we use the notion of \emph{attack sequences} (see Definition~\ref{attack-seq}) in the proof.
\begin{proposition}~\label{theorem-stage}
Let $\sigma_{CF2}$ be CF2 semantics.
    For every two argumentation frameworks $AF = (AR, Attacks)$ and $AF' =  (AR', Attacks')$, such that $AF \preceq_N AF'$, the following statement holds true:
    \begin{align*}
        \forall E \in \sigma_{CF2}(AF),  (\exists E' \in \sigma_{CF2}(AF'),  \text{such that } (E' \not \subseteq AR \lor E' = E))
    \end{align*}
\end{proposition}
\begin{proof}\label{ref-dep-p3}
    \phantom{eee}
    \begin{enumerate}
        \item In the proof of Proposition~\ref{naive-ref-dep}, we have shown that naive semantics satisfies weak reference independence. It follows that by definition of CF2 semantics as an SCC-recursive and naive-based semantics, $\forall E \in \sigma_{CF2}(AF), \exists E' \in \sigma_{CF2}(AF')$, such that the following statement holds true:
        \begin{align*}
            &{} \text{if } \exists SCC \in SCCS_{AF}, \text{ such that } (E \cap SCC) \neq (E' \cap SCC), \\
            &{} \text{then }  \exists V_{a, b} \in {\cal V}(AF'), \\ 
            &{} \quad \text{ such that } a \in AR' \setminus AR, a \in E',  b \in (E \cap SCC), V_{b, a} \not \in {\cal V}(AF')
        \end{align*}
        \item Now, let us observe that $\forall E' \in \sigma_{CF2}(AF')$, (if $(\exists V_{c, d} \in {\cal V}(AF'), V_{c, d} \not \in {\cal V}(AF), d \in AR', c \in AR' \setminus AR, c \in E')$ then $E' \not \subseteq AR$). Consequently, from 1. it follows that $\forall E \in \sigma_{CF2}(AF), \exists E' \in \sigma_{CF2}(AF')$ such that the following statement holds true:
        \begin{align*}
            &{} \forall SCC \in SCCS_{AF}, (E \cap SCC) = (E' \cap SCC) \\
            &{} \lor E' \not \subseteq AR
        \end{align*}
        \item By definition of CF2 semantics as an SCC-recursive, naive-based semantics, given $E \in \sigma_{CF2}(AF)$, if $(\exists E' \in \sigma_{CF2}(AF')$ such that $\forall SCC \in SCCS_{AF}, (E \cap SCC) = (E' \cap SCC))$ then $E = E'$ or $E' \not \subseteq AR$. Consequently, it follows from 2. that $\forall E \in \sigma_{CF2}(AF), \exists E' \in \sigma_{CF2}(AF')$ such that the following statement holds true:
        \begin{align*}
            E = E' \lor E' \not \subseteq AR
        \end{align*}
        We have proven the proposition.
    \end{enumerate}
\end{proof}

\section{Cycles and Reference Independence}
\label{rational-expansion}
In the previous section we have shown that only one of the evaluated argumentation semantics -- CF2 semantics -- satisfies the weak reference independence principle.
In order to guarantee economic rationality, and in particular reference independence, it is relevant to look beyond argumentation semantics.
At first glance, it is striking that the example expansions in the proofs of Propositions~\ref{irrational-1},~\ref{irrational-2}, and~\ref{irrational-3} add new cycles to the argumentation graphs.
Consequently, we examine if an argumentation framework expansion can be defined that can guarantee compliance with the rational man's argumentation principle by further restricting the relationship of two argumentation frameworks $AF \preceq_N AF'$.
For this, we first introduce a definition of cycles in the context of argumentation frameworks.
\begin{definition}[Attack Cycles in Argumentation Frameworks]
    An attack cycle $C$ in an argumentation framework $AF = (AR, Attacks)$ is an argument sequence $\langle a_1, a_2, ..., a_{n-1}, a_n \rangle$ such that $(a_{i}, a_{i+1}) \in Attacks$, $a_i \in AR$, $1 \leq i \leq n-1$, $a_1 = a_n$, and $|AR^{V}_{a_1, a_n}| = n-1$. ${\cal C}(AF)$ denotes all attack cycles of AF and $AR^{C}$ denotes the arguments that occur in an attack cycle $C$.
\end{definition}
Now, we define the concept of a \emph{non-cyclic expansion}.
\begin{definition}[Non-Cyclic Expansion]
    A non-cyclic expansion of two argumentation frameworks $AF = (AR, Attacks)$ and $AF' = (AR', Attacks')$ (denoted by $AF \preceq_{NC} AF'$) is an expansion $AF \preceq_E AF'$, for which it holds true that ${\cal C}(AF') = {\cal C}(AF)$.
    
\end{definition}
\noindent This allows us to define the \emph{rational man's expansion}.
\begin{definition}[Rational Man's Expansion]
    A rational man's argumentation expansion of two argumentation frameworks $AF = (AR, Attacks)$ and $AF' = (AR', Attacks')$ (denoted by $AF \preceq_{RM} AF'$) is an expansion $AF \preceq_E AF'$, for which the following conditions hold true:
    \begin{enumerate}
        \item $AF \preceq_{N} AF'$;
        \item $AF \preceq_{NC} AF'$;
        \item $\forall a \in AR, b \in AR' \setminus AR$, such that $b$ is reachable from $a$, it holds true that $\forall C \in {\cal C}(AF'), a \not \in AR^{'C}$.
    \end{enumerate}
\end{definition}
Colloquially speaking, a rational man's expansion is a normal expansion in which \emph{i)} no additional cycles are added to the initial argumentation framework and no additional arguments are added to existing cycles, and \emph{ii)} no newly added arguments are reachable from cycles.

We prove that the \emph{rational man's expansion} guarantees weak reference independence, given complete, preferred, stage2, CF2, grounded, ideal, and eager argumentation semantics.
The proof relies on the observation that given any argumentation framework, each argument that is in a strongly admissible set is by definition always also in at least one extension of any of the aforementioned argumentation semantics.

\begin{proposition}
    Let $\sigma_{x}$ be an argumentation semantics, such that $x \in \{complete, preferred, semi-stable, stage, \\ stage2, CF2, grounded, ideal, eager\}$.
    For every two argumentation frameworks $AF = (AR, \\ Attacks)$ and $AF' = (AR',  Attacks')$, such that $AF \preceq_{RM} AF'$ the following statement holds true:
    \begin{align*}
        \forall E \in \sigma_x(AF), \exists E' \in \sigma_x(AF')
        \text{ such that } E' \not \subseteq AR \lor E' = E
    \end{align*}
\end{proposition}
\begin{proof}
We provide a proof by contradiction.
    \begin{enumerate}
        \item Let us suppose the following statement holds true:
    \begin{align*}
        \exists E \in \sigma_x(AF), \text{ such that } \forall E' \in \sigma_x(AF')
        \text{ it holds true that } E' \subseteq AR \land E' \neq E
    \end{align*}
    \item Let us observe that by definition of $\sigma_x$ and because $AF \preceq_{RM} AF'$ implies $\forall a \in AR' \setminus AR$, $\forall C \in {\cal C}(AF')$, $a \not \in C$ and a is not reachable from any $c \in AR'_C$, it follows from \emph{1.} that the following statement holds true:
    \begin{align*}
        &{} \exists b \in AR' \setminus AR \text{ such that } b \text{ attacks } AR \text{ and } \\
        &{} b \in S \subseteq AR', \text{ such that } S \text{ is strongly admissible and } \forall E' \in \sigma(AF'), b \not \in E'
    \end{align*}
    \item By definition of $\sigma_x$, iff $\exists c \in AR'$, such that $c \in S \subseteq AR'$ and $S$ is strongly admissible, then $\exists E' \in \sigma(AF'), c \in E'$.
    Hence, $\exists E' \in \sigma_{x}(AF')$, such that $b \in E'$. Note that $b \in AR' \setminus AR$. Contradiction.
    \end{enumerate}
\end{proof}

Let us note that allowing for expansions that add \emph{even} cycles to an argumentation framework cannot guarantee reference independence in the case of many argumentation semantics.
For example, given $AF = (\{a, b\}, \{(a, b)\})$ and its expansion $AF' = (\{a, b, c, d\}, \{(a, b), (b, c), \\(c, d), (d, a)\}$, grounded, ideal, and eager semantics, as well as the skeptical reasoning modes of complete, preferred, semi-stable, stage, stage2, and CF2 semantics return $\{\{a\}\}$ for $AF$ and $\{\{\}\}$ for $AF'$, as depicted in Figure~\ref{fig:cycle-examples-1}.
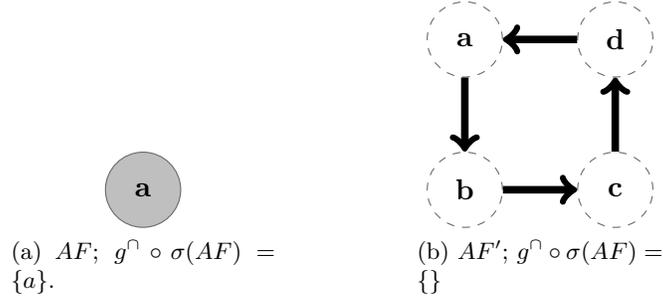
\begin{figure}
    \subfloat[$AF$; $g^{\cap} \circ \sigma(AF) = \{a\}$.]{
        \begin{tikzpicture}[
            noanode/.style={dashed, circle, draw=black!60, minimum size=10mm, font=\bfseries},
            anode/.style={circle, fill=lightgray, draw=black!60, minimum size=10mm, font=\bfseries},
            ]
            \node[draw=none]    (NV1)    at(0,4)  {};
            \node[anode]    (A)    at(1.5,4)  {a};
            \node[draw=none]    (NV2)    at(3,4)  {};
        \end{tikzpicture}
    }
    \hspace{50pt}
    \centering
    \subfloat[$AF'$; $g^{\cap} \circ \sigma(AF) = \{\}$]{
        \begin{tikzpicture}[
            noanode/.style={dashed, circle, draw=black!60, minimum size=10mm, font=\bfseries},
            anode/.style={circle, fill=lightgray, draw=black!60, minimum size=10mm, font=\bfseries},
            ]
            \node[noanode]    (A)    at(0,2)  {a};
            \node[noanode]    (B)    at(0,0)  {b};
            \node[noanode]    (C)    at(2,0)  {c};
            \node[noanode]    (D)    at(2,2)  {d};
            \path [->, line width=1mm]  (A) edge node[left] {} (B);
            \path [->, line width=1mm]  (B) edge node[left] {} (C);
            \path [->, line width=1mm]  (C) edge node[left] {} (D);
            \path [->, line width=1mm]  (D) edge node[left] {} (A);
        \end{tikzpicture}
    }
\caption{Importance of absence of even-length cycles, given, \emph{e.g.}, any of the surveyed semantics and skeptical acceptance. $AF \not \preceq_{RM} AF'$.}
\label{fig:cycle-examples-1}
\end{figure}
Also, it is not sufficient that only cycles that include at least one argument $arg \in AR$ and at least one argument $arg' \in AR' \setminus AR$ are not allowed in an expansion $AF \preceq_{NC} AF'$.
This can be shown by introducing the following example. Given $AF = (\{a\}, \{\})$ and $AF' = (\{a, b, c, d\}, \{(b, a), (c, b), (d, c), (b, d)\})$, grounded, ideal, and eager semantics, as well as \emph{skeptical} complete, preferred, semi-stable, stage, stage2, and CF2 semantics return $\{\{a\}\}$ for $AF$ and $\{\{\}\}$ for $AF'$, as depicted in Figure~\ref{fig:cycle-examples-2}.
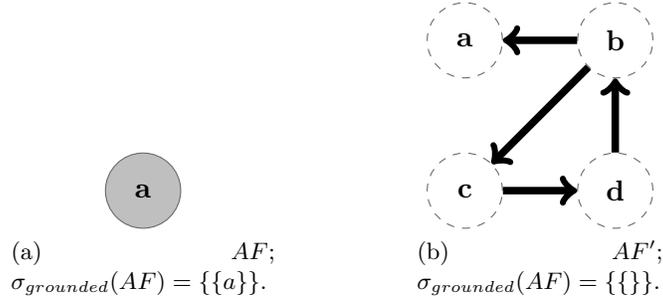
\begin{figure}
    \subfloat[$AF$; $\sigma_{grounded}(AF) = \{\{a\}\}$.]{
        \begin{tikzpicture}[
            noanode/.style={dashed, circle, draw=black!60, minimum size=10mm, font=\bfseries},
            anode/.style={circle, fill=lightgray, draw=black!60, minimum size=10mm, font=\bfseries},
            ]
            \node[draw=none]    (NV1)    at(0,4)  {};
            \node[anode]    (A)    at(1.5,4)  {a};
            \node[draw=none]    (NV2)    at(3,4)  {};
        \end{tikzpicture}
    }
    \hspace{50pt}
    \centering
    \subfloat[$AF'$; $\sigma_{grounded}(AF) = \{\{\}\}$.]{
        \begin{tikzpicture}[
            noanode/.style={dashed, circle, draw=black!60, minimum size=10mm, font=\bfseries},
            anode/.style={circle, fill=lightgray, draw=black!60, minimum size=10mm, font=\bfseries},
            ]
            \node[noanode]    (A)    at(0,2)  {a};
            \node[noanode]    (B)    at(2,2)  {b};
            \node[noanode]    (C)    at(0,0)  {c};
            \node[noanode]    (D)    at(2,0)  {d};
            \path [->, line width=1mm]  (B) edge node[left] {} (A);
            \path [->, line width=1mm]  (B) edge node[left] {} (C);
            \path [->, line width=1mm]  (C) edge node[left] {} (D);
            \path [->, line width=1mm]  (D) edge node[left] {} (B);
        \end{tikzpicture}
    }
\caption{Importance of absence of cycles among arguments $Args \subseteq (AR' \setminus AR)$, given, \emph{e.g.} grounded semantics. $AF \not \preceq_{RM} AF'$.}
\label{fig:cycle-examples-2}
\end{figure}
Given the presented findings, it is obvious that if \emph{expansion} and \emph{deletion} (the removal of arguments) of an argumentation framework are allowed in any argumentation scenario, cycles should be avoided altogether.
For example, assuming preferred semantics, the argumentation framework $AF = (\{a, b, c\}, \{(a, b), (b, c), (c, a)\})$ implies (among others) the preference $\{ \} \succeq \{ a \}$. Removing the argument $c$ from $AF$ gives us $AF' = (\{a, b\}, \{(a, b)\}) $, with $AF' \preceq_{N} AF$.
As $AF'$ implies (among others) $\{ a \} \succeq \{ \}$, the preference relation implied by $AF$ is inconsistent with the preference relation implied by $AF'$.
While it is clear that it is, in many scenarios, impractical to avoid cycles (or cyclic expansions) altogether, the findings highlight the importance of further research on \textquote{loop-busting} methods to ensure reference independence.

\section{Rational Man's Argumentation, Belief Revision, and Dialogues}
\label{dialoges}
To highlight the relevance of the presented research, this section provides two examples that illustrate how the newly established principles and expansions can be applied.

\begin{example}[Argumentation dialogues]
    In a multi-agent scenario, let us assume we have a decision-maker agent ${\cal A}_1$ that receives advise from a consultant agent ${\cal A}_2$.
    In this context, ${\cal A}_1$ presents its argumentation framework $AF = (AR, Attacks)$ to ${\cal A}_2$, who then proposes changes by providing $AF' = (AR', Attacks')$, with $AF'$ being a normal expansion of $AF$ ($AF \preceq_N AF'$).
    Subsequently, ${\cal A}_1$ can accept or reject the changes.
    ${\cal A}_2$ can propose two types of changes:
    \begin{description}
        \item[Set-expanding changes.] ${\cal A}_2$ only shows ${\cal A}_1$ that additional options to choose from exist and how these options should be integrated into the argumentation framework.
        \item[Belief-revising changes.] ${\cal A}_2$ advises ${\cal A}_1$ to change its beliefs about the decision options contained in $AF$ and may in addition propose set-expanding changes.
    \end{description}
    ${\cal A}_2$ might want to deceive ${\cal A}_1$ by proposing changes that ${\cal A}_2$ labels as \emph{set-expanding} but that are also \emph{belief-revising}.
    For example, ${\cal A}_1$ presents the following argumentation framework to ${\cal A}_2$:
    \begin{align*}
        AF = (\{a, b, c\}, \{(b, a), (c, a)\}))
    \end{align*}
    Then, ${\cal A}_2$ proposes the following:
    \begin{align*}
        AF' = (\{a, b, c, d, e, f\}, \{(b, a), (c, a), (d, c), (e, d) (f, e) (d, f)\}))
    \end{align*}
    If ${\cal A}_2$ labels this proposal as \emph{set-expanding}, ${\cal A}_2$ is deceiving ${\cal A}_1$ to think that ${\cal A}_1$ is barely considering new options and not revising the assessment of the existing options.
    However, given the work presented above, ${\cal A}_1$ can first detect that $AF \npreceq_{RM} AF'$ ($AF \npreceq_{NC} AF'$) and then check if its preferences over the initial arguments in $AF$ are still consistent, given the decision that is made based on $AF'$. As this is not the case, ${\cal A}_1$ can, for example, decide to stick to the initial decision that it made based on $AF$.
    Figure~\ref{fig:final-example-2} depicts the argumentation graphs of $AF$ and $AF'$.
    \begin{figure}
        \subfloat[$AF$; $g^{\cap} \circ \sigma(AF) = \{b, c\}$.]{
            \begin{tikzpicture}[
                noanode/.style={dashed, circle, draw=black!60, minimum size=10mm, font=\bfseries},
                anode/.style={circle, fill=lightgray, draw=black!60, minimum size=10mm, font=\bfseries},
                ]
                \node[draw=none]    (NV1)    at(-0,5,4)  {};
                \node[noanode]    (A)    at(0,2)  {a};
                \node[anode]    (B)    at(0,4)  {b};
                \node[anode]    (C)    at(2,4)  {c};
                \node[draw=none]    (NV2)    at(2.5,4)  {};
                \path [->, line width=1mm]  (B) edge node[left] {} (A);
                \path [->, line width=1mm]  (C) edge node[left] {} (A);
            \end{tikzpicture}
        }
        \hspace{50pt}
        \centering
        \subfloat[$AF'$; $g^{\cap} \circ \sigma(AF') = \{b\}$.]{
            \begin{tikzpicture}[
                noanode/.style={dashed, circle, draw=black!60, minimum size=10mm, font=\bfseries},
                anode/.style={circle, fill=lightgray, draw=black!60, minimum size=10mm, font=\bfseries},
                ]
                \node[noanode]    (A)    at(0,2)  {a};
                \node[anode]    (B)    at(0,4)  {b};
                \node[noanode]    (C)    at(2,4)  {c};
                \node[noanode]    (D)    at(2,2)  {d};
                \node[noanode]    (E)    at(4,2)  {e};
                \node[noanode]    (F)    at(4,4)  {f};
                \path [->, line width=1mm]  (B) edge node[left] {} (A);
                \path [->, line width=1mm]  (C) edge node[left] {} (A);
                \path [->, line width=1mm]  (D) edge node[left] {} (C);
                \path [->, line width=1mm]  (E) edge node[left] {} (D);
                \path [->, line width=1mm]  (F) edge node[left] {} (E);
                \path [->, line width=1mm]  (D) edge node[left] {} (F);
            \end{tikzpicture}
        }
    \caption{$AF \preceq_N AF'$, but $AF \npreceq_{NC} AF'$.} %
    \label{fig:final-example-2}
    \end{figure}
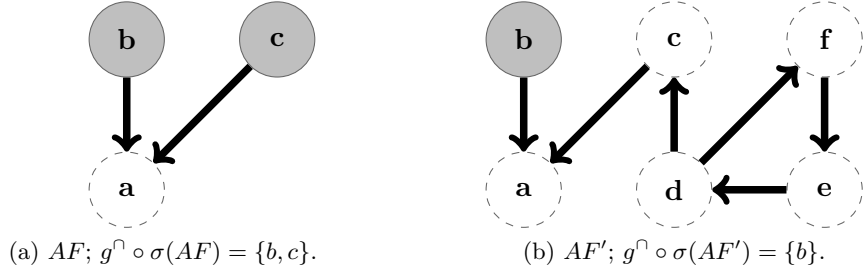
\end{example}
\begin{example}[Argument Mining]
    The ability to assess whether argumentation-based inference is reference independent can potentially be useful in argument mining scenarios, in which argumentation graphs are generated from natural language text or other data sources~\cite{Lippi:2016:AMS:2909066.2850417}.
    Let us introduce a scenario where an argument miner uses machine learning techniques for natural language processing to generate argumentation frameworks from text -- for example, from legal documents -- and then hands them over to an agent that resolves the argumentation frameworks to inform its decision-making.
    However, the argumentation agent is not accepting the frameworks under any condition; instead, it is assessing the frameworks and their relation with each other to determine if reference independence is violated.
    The argumentation agent then provides the results of these assessments to the argument miner, who can use the information in different ways.
    If reference independence is violated, it can either re-assess the corresponding text and suggest an alternative, reference independent interpretation, or consider the text as not useful and label it accordingly to increase its ability to focus on more useful texts in the future. Figure~\ref{fig:architecture} depicts the architecture of the proposed system.
    \begin{figure}
        \centering
        \includegraphics[width=0.4\textwidth]{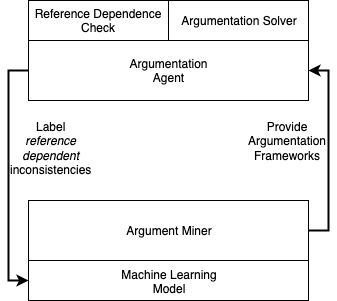}
        \caption{Labeling Reference-dependent Argumentation Graphs to improve Argument Mining.} 
        \label{fig:architecture}
    \end{figure}
    As an example, let us assume the argument miner creates argumentation frameworks based on an evolving online discussion on whether a policy should be implemented or not (denoted by argument $p$).
    At time $t_0$ the argument miner detects an argument $a$ that attacks the policy implementation proposal: $AF_0 = (\{p, a\}, \{(a, p)\})$.
    The argumentation agent -- using, for example, skeptical complete semantics $\sigma^{\cap}_{complete}$ -- resolves $AF$ to $\{ a \}$, \emph{i.e.}, it decides the policy should not be implemented.
    At time $t_1$, the argument miner detects the additional arguments $b$ and $c$, as well as the additional attacks $(a, b)$, $(b, c)$, and $(c, a)$.
    The argumentation agent resolves the framework $AF_1 = (\{p, a, b, c\}, \{(a, p), (a, b), (b, c), (c, a)\})$ as $\{ \}$.
    Now, it is clear that $AF_0 \preceq_{N} AF_1$, but $AF_0 \not \preceq_{RM} AF_1$.
    It is also clear that the preference relations implied by $g \circ \sigma^{\cap}_{complete}(AF_0)$ and $g \circ \sigma^{\cap}_{complete}(AF_0)$ are inconsistent, \emph{i.e.}, $g^{\cap} \circ \sigma^{\cap}_{complete}(AF_0)$ implies $\{ a \} \succeq \{ \}$ and $g^{\cap} \circ \sigma^{\cap}_{complete}(AF_1)$ implies $\{ \} \succeq \{ a \}$.
    Hence, the argumentation agent can label $AF_1$ as faulty or not useful and provide this information to the argument miner, who can then attempt to find alternative formal interpretations of the discussion, or move on to a different discussion.
    Figure~\ref{fig:mining-example} depicts the argumentation graphs of $AF_0$ and $AF_1$.
    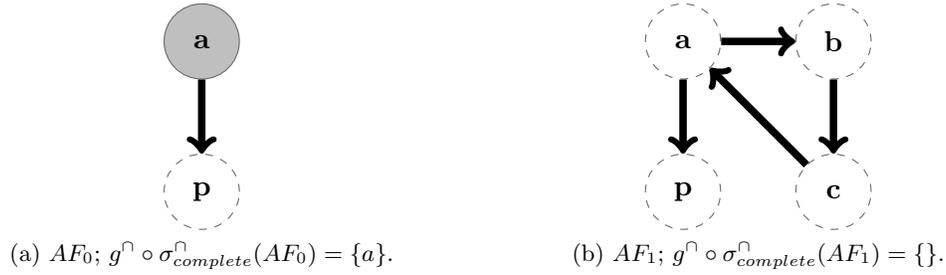
\begin{figure}
        \subfloat[$AF_0$; $g^{\cap} \circ \sigma^{\cap}_{complete}(AF_0) = \{ a \}$.]{
            \begin{tikzpicture}[
                noanode/.style={dashed, circle, draw=black!60, minimum size=10mm, font=\bfseries},
                anode/.style={circle, fill=lightgray, draw=black!60, minimum size=10mm, font=\bfseries},
                ]
                \node[draw=none]    (NV1)    at(-2.5,4)  {};
                \node[noanode]    (P)    at(0,2)  {p};
                \node[anode]    (A)    at(0,4)  {a};
                \node[draw=none]    (NV2)    at(2.5,4)  {};
                \path [->, line width=1mm]  (A) edge node[left] {} (P);
            \end{tikzpicture}
        }
        \hspace{50pt}
        \centering
        \subfloat[$AF_1$; $g^{\cap} \circ \sigma^{\cap}_{complete}(AF_1) = \{ \}$.]{
            \begin{tikzpicture}[
                noanode/.style={dashed, circle, draw=black!60, minimum size=10mm, font=\bfseries},
                anode/.style={circle, fill=lightgray, draw=black!60, minimum size=10mm, font=\bfseries},
                ]
                \node[draw=none]    (NV1)    at(-1.5,4)  {};
                \node[noanode]    (P)    at(0,2)  {p};
                \node[noanode]    (A)    at(0,4)  {a};
                \node[noanode]    (B)    at(2,4)  {b};
                \node[noanode]    (C)    at(2,2)  {c};
                \node[draw=none]    (NV2)    at(3.5,4)  {};
                \path [->, line width=1mm]  (A) edge node[left] {} (P);
                \path [->, line width=1mm]  (A) edge node[left] {} (B);
                \path [->, line width=1mm]  (B) edge node[left] {} (C);
                \path [->, line width=1mm]  (C) edge node[left] {} (A);
            \end{tikzpicture}
        }
    \caption{$AF_0 \preceq_N AF_1$, but $AF_0 \npreceq_{RM} AF_1$ ($AF_0 \npreceq_{NC} AF_1$).}
    \label{fig:mining-example}
    \end{figure}
\end{example}

\section{Related Work: Preference-based Argumentation and Rational Man's Expansions}
\label{related}
In our work, we derive implicit preference orders on sets of arguments from abstract argumentation frameworks.
Hence, it makes sense to put our work in the context of argumentation approaches that \emph{explicitly} define preferences.
Amgoud's and Cayrol's \emph{preference-based argumentation}~\cite{Amgoud2002} can be considered the most foundational work that advances this research direction.
Hence, we relate our work to preference-based argumentation and confirm the intuition that the explicit definition of preferences does not guarantee reference independence by formal proof.
Let us first introduce a definition of a preference-based argumentation framework.
\begin{definition}[Preference-based Argumentation Framework~\cite{Amgoud2002}]
    A preference-based argumentation framework is a triplet $(AR, Attacks, Prefs)$, whereby $AR$ and $Attacks$ are arguments and attacks, defined according to Definition~\ref{arg-f} and $Prefs$ define a partial or total ordering on $AR \times AR$.
\end{definition}

In a preference-based argumentation framework, acceptability is determined as follows.
\begin{definition}[Preference-based Argumentation Framework~\cite{Amgoud2002}]
    Given a preference-based argumentation framework $AF_p = (AR, Attacks, Prefs)$, the set of acceptable arguments $Args_{acc} \subseteq AR$ is determined as follows by the preference-based argumentation function $\tau_{preferred}$:
    \begin{align*}
        \tau_{preferred}(AF) = \{ a | a \in AR, \text{ such that } \forall b \in AR, \text{ if } (b, a) \in Attacks \text{ then } a \succeq b \}
    \end{align*}
\end{definition}
Note that in this paper, we only consider the \emph{acceptability} status of arguments; \emph{i.e.}, in contrast to Amgoud and Cayrol, we do not distinguish between \emph{rejected arguments} and \emph{arguments in abeyance}. This simplification is motivated by the rational decision-maker's required ability to make clear, unambiguous decisions\footnote{Analogously, we also do not distinguish between rejected and undecided arguments in abstract argumentation.}.

To analyze preference-based argumentation in the context of reference independence, let us first define the concept of a normal expansion of a preference-based argumentation framework.
\begin{definition}
    An expansion $AF'_p = (AR', Attacks', Prefs')$ of a preference-based argumentation framework $AF_p = (AR, Attacks, Prefs)$ is \emph{normal} ($AF_p \preceq_{NP} AF'_p$) iff:
    \begin{itemize}
        \item $Attacks \subseteq Attacks'$ and
        \item $\forall (a, b) \in Attacks' \setminus Attacks, a \in AR' \setminus AR \lor b \in AR' \setminus AR$ and
        \item $Prefs \subseteq Prefs'$ and 
        \item $\forall (a \succeq b) \in Prefs' \setminus Prefs, a \in AR' \setminus AR \lor b \in AR' \setminus AR$.
    \end{itemize} 
\end{definition}
In words, considering the addition of preferences $Prefs$ to abstract argumentation frameworks, we assume that a normal expansion $AF'_p$ of $AF_p$ neither changes existing preferences defined in $Prefs$ nor adds additional preferences between any two arguments that exist in $AF_p$.
As non-cyclic expansions do not require a definition that is specific to preference-based argumentation, let us directly define the rational man's expansion in the context of preference-based argumentation.
\begin{definition}
    An expansion $AF'_p = (AR', Attacks', Prefs')$ of a preference-based argumentation framework $AF_p = (AR, Attacks, Prefs)$ is a \emph{rational man's expansion} ($AF_p \preceq_{RMP} AF'_p$) iff:
    \begin{itemize}
        \item $AF_p \preceq_{NP} AF'_p$
        \item $AF \preceq_{NC} AF'$
        \item $\forall a \in AR_a, b \in AR_a' \setminus AR_a$ such that $b$ is reachable from $a$, it holds true that $\forall C \in {\cal C}(AF'), a \not \in AR_a^{'C}$,
    \end{itemize}
    where $AF = (AR_a, Attacks_a)$ and $AF' = (AR'_a, Attacks'_a)$ are abstract argumentation frameworks, such that $AR_a = AR, Attacks_a = Attacks, AR'_a = AR', Attacks' = Attacks$.
\end{definition}
Note that in the definition above, we \textquote{map} preference-based argumentation frameworks to abstract argumentation frameworks, in order to be able to use definitions we have introduced for abstract argumentation frameworks.

Now, it can be easily shown that preference-based argumentation does not guarantee reference independence.
For this, we introduce the following proposition.
\begin{proposition}
    Let $g^{\cap} \circ \tau_{preferred}$ be an argumentation-based decision function for preference-based argumentation frameworks.
    For every two preference-based argumentation frameworks $AF_p$ and $AF'_p$, such that $AF_p = (AR, Attacks, Prefs)$, $AF'_p = (AR', Attacks', Prefs')$, and $AF_p \preceq_{NP} AF'_p$, the following statement holds true:
    \begin{align*}
        g^{\cap} \circ \tau_{preferred}(AF'_p) \subseteq AR \text{ does not imply } g^{\cap} \circ \tau_{preferred}(AF'_p) = g^{\cap} \circ \tau_{preferred}(AF_p)
    \end{align*}
\end{proposition}
\begin{proof}
    The proposition can be proven by counter-example.
    We introduce the following preference-based argumentation frameworks:
    \begin{itemize}
        \item $AF_p = (\{a, b, c\}, \{(a, b)\}, (a \succeq c, b \succeq c)\})$;
        \item $AF'_p = (\{a, b, c, d\}, \{(a, b), (b, d), (d, a)\}, (a \succeq c, b \succeq c, d \succeq c))\}$.
    \end{itemize}
    We can see that $AF_p \preceq_{NP} AF'_p$. $g^{\cap} \circ \tau_{preferred}$ resolves the frameworks as follows:
    \begin{enumerate}
        \item $g^{\cap} \circ \tau_{preferred}(AF_p) = \{a, c\}$.
        \item $g^{\cap} \circ \tau_{preferred}(AF'_p) = \{c\}$.
    \end{enumerate}
    From $g^{\cap} \circ \tau_{preferred}(AF'_p) \subseteq AR$ and $g^{\cap} \circ \tau_{preferred}(AF'_p) \neq g^{\cap} \circ \tau_{preferred}(AF_p)$ it follows that $g^{\cap} \circ \tau_{preferred}(AF'_p) \subseteq AR \text{ does not imply } g^{\cap} \circ \tau_{preferred}(AF'_p) = g^{\cap} \circ \tau_{preferred}(AF_p)$.
\end{proof}
Figure~\ref{fig:vale-pref-1} depicts the argumentation frameworks used in the proof.
\begin{figure}
    \subfloat[$AF_p$, with preferences $(a \succeq c, b \succeq c)$; $g^{\cap} \circ \tau_{preferred}(AF) = \{a, c\}$.]{
        \begin{tikzpicture}[
            noanode/.style={dashed, circle, draw=black!60, minimum size=10mm, font=\bfseries},
            anode/.style={circle, fill=lightgray, draw=black!60, minimum size=10mm, font=\bfseries},
            ]
            \node[draw=none]    (NV1)    at(-0,5,4)  {};
            \node[anode]    (A)    at(0,2)  {a};
            \node[noanode]    (B)    at(0,4)  {b};
            \node[anode]    (C)    at(2,4)  {c};
            \node[draw=none]    (NV2)    at(2.5,4)  {};
            \path [->, line width=1mm]  (A) edge node[left] {} (B);
        \end{tikzpicture}
    }
    \hspace{50pt}
    \centering
    \subfloat[$AF'_p$, with preferences $(a \succeq c, b \succeq c, d \succeq c)$; $g^{\cap} \circ \tau_{preferred}(AF'_p) = \{ c \}$.]{
        \begin{tikzpicture}[
            noanode/.style={dashed, circle, draw=black!60, minimum size=10mm, font=\bfseries},
            anode/.style={circle, fill=lightgray, draw=black!60, minimum size=10mm, font=\bfseries},
            ]
            \node[noanode]    (A)    at(0,2)  {a};
            \node[noanode]    (B)    at(0,4)  {b};
            \node[anode]    (C)    at(2,4)  {c};
            \node[noanode]    (D)    at(2,2)  {d};
            \path [->, line width=1mm]  (A) edge node[left] {} (B);
            \path [->, line width=1mm]  (B) edge node[left] {} (D);
            \path [->, line width=1mm]  (D) edge node[left] {} (A);
        \end{tikzpicture}
    }
\caption{Normal, but cyclic expansion: $AF_p \preceq_{NP} AF'_p$, but $AF_p \npreceq_{NC} AF'_p$.}
\label{fig:vale-pref-1}
\end{figure}
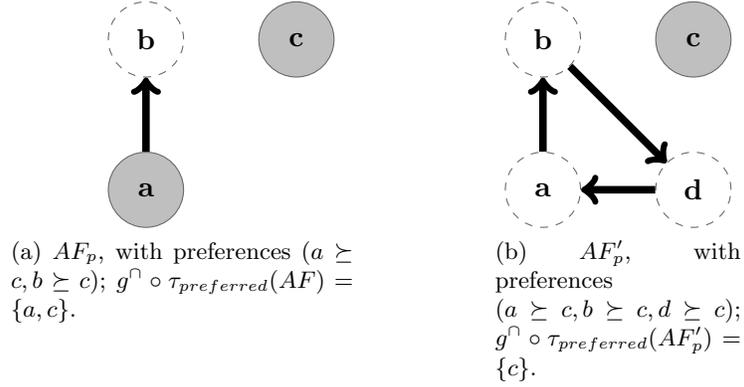

Also, it can be proven that given two preference-based argumentation frameworks $AF_p$ and $AF'_p$ even when $AF'_p$ is a rational man's expansion of $AF_p$ ($AF_p \preceq_{RMP} AF'_p$), reference dependence is not guaranteed.
\begin{proposition}
    Let $g^{\cap} \circ \tau_{preferred}$ be any argumentation-based decision function for preference-based argumentation frameworks.
    For every two preference-based argumentation frameworks $AF_p$ and $AF'_p$, such that $AF_p = (AR, Attacks, Prefs)$, $AF'_p = (AR', Attacks', Prefs')$, and $AF_p \preceq_{RMP} AF'_p$, the following statement holds true:
    \begin{align*}
        g^{\cap} \circ \tau_{preferred}(AF'_p) \subseteq AR \text{ does not imply } g^{\cap} \circ \tau_{preferred}(AF'_p) = g^{\cap} \circ \tau_{preferred}(AF_p)
    \end{align*}
\end{proposition}
\begin{proof}
    Again, the proposition can be proven by counter-example.
    We introduce the following preference-based argumentation frameworks:
    \begin{itemize}
        \item $AF_p = (\{a, b, c\}, \{(a, b)\}, (a \succeq c, b \succeq c)\})$;
        \item $AF'_p = (\{a, b, c, d\}, \{(a, b), (c, d), (d, a)\}, (a \succeq c, b \succeq c, d \succeq c))\}$.
    \end{itemize}
    We can see that $AF_p \preceq_{RMP} AF'_p$. $g^{\cap} \circ \tau_{preferred}$ resolves the frameworks as follows:
    \begin{enumerate}
        \item $g^{\cap} \circ \tau_{preferred}(AF_p) = \{a, c\}$.
        \item $g^{\cap} \circ \tau_{preferred}(AF'_p) = \{c\}$.
    \end{enumerate}
    From $g^{\cap} \circ \tau_{preferred}(AF'_p) \subseteq AR$ and $g^{\cap} \circ \tau_{preferred}(AF'_p) \neq g^{\cap} \circ \tau_{preferred}(AF_p)$ it follows that $g^{\cap} \circ \tau_{preferred}(AF'_p) \subseteq AR \centernot \implies g^{\cap} \circ \tau_{preferred}(AF'_p) = g^{\cap} \circ \tau_{preferred}(AF_p)$.
\end{proof}
The argumentation frameworks used in the proof are depicted by Figure~\ref{fig:vale-pref-2}.
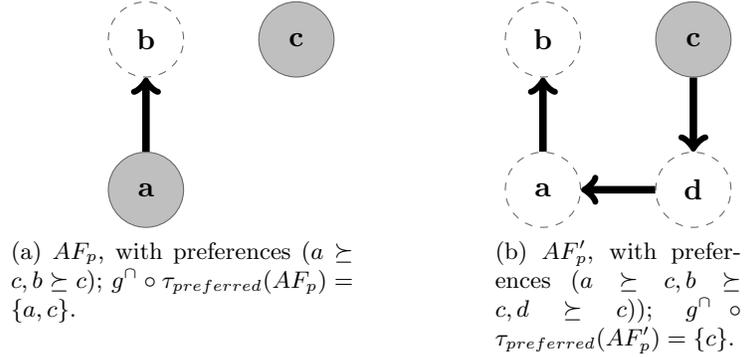
\begin{figure}
    \subfloat[$AF_p$, with preferences $(a \succeq c, b \succeq c)$; $g^{\cap} \circ \tau_{preferred}(AF_p) = \{a, c\}$.]{
        \begin{tikzpicture}[
            noanode/.style={dashed, circle, draw=black!60, minimum size=10mm, font=\bfseries},
            anode/.style={circle, fill=lightgray, draw=black!60, minimum size=10mm, font=\bfseries},
            ]
            \node[draw=none]    (NV1)    at(-0,5,4)  {};
            \node[anode]    (A)    at(0,2)  {a};
            \node[noanode]    (B)    at(0,4)  {b};
            \node[anode]    (C)    at(2,4)  {c};
            \node[draw=none]    (NV2)    at(2.5,4)  {};
            \path [->, line width=1mm]  (A) edge node[left] {} (B);
        \end{tikzpicture}
    }
    \hspace{50pt}
    \centering
    \subfloat[$AF'_p$, with preferences $(a \succeq c, b \succeq c, d \succeq c))$; $g^{\cap} \circ \tau_{preferred}(AF'_p) = \{ c \}$.]{
        \begin{tikzpicture}[
            noanode/.style={dashed, circle, draw=black!60, minimum size=10mm, font=\bfseries},
            anode/.style={circle, fill=lightgray, draw=black!60, minimum size=10mm, font=\bfseries},
            ]
            \node[noanode]    (A)    at(0,2)  {a};
            \node[noanode]    (B)    at(0,4)  {b};
            \node[anode]    (C)    at(2,4)  {c};
            \node[noanode]    (D)    at(2,2)  {d};
            \path [->, line width=1mm]  (A) edge node[left] {} (B);
            \path [->, line width=1mm]  (C) edge node[left] {} (D);
            \path [->, line width=1mm]  (D) edge node[left] {} (A);
        \end{tikzpicture}
    }
\caption{Violation of reference independence in preference-based argumentation in absence of cycles.}
\label{fig:vale-pref-2}
\end{figure}
As preference-based argumentation is a special case of value-based argumentation as introduced by Bench-Capon~\cite{bench2003persuasion}, it is also clear that normal, non-cyclic expansions cannot guarantee rationality for value-based argumentation.
Let us highlight that in contrast to these works, which introduce explicit preference relations on arguments, our work analyzes argumentation in the context of the implicit preference relations on sets of arguments that are implied by argumentation semantics.

\section{Conclusion and Future Work}
\label{conclusion}
In this work, we observe that abstract argumentation semantics imply preferences orders on sets of arguments.
Hence, we explore these preference orders from the perspective of microeconomic theory, and in particular the \emph{rational man} paradigm.
This paper shows that abstract argumentation semantics typically do not guarantee reference independent decision-making according to the \emph{rational man paradigm} and consequently uncovers a gap between abstract argumentation semantics and economic rationality.
Thereby, our research focuses on foundational work at the intersection of abstract argumentation and (bounded) economic rationality.
Our work sheds new light on the question to what extent and under which circumstances the application of abstract argumentation returns \textquote{reasonable} results from the perspective of economic rationality.
Taking into account the rich body of formal works on both argumentation and (boundedly) rational decision-making, plenty of opportunities to extend our work exist.
In particular, we consider the following research directions as promising future work:
\begin{itemize}
    \item \textbf{'Loop-busting' to ensure economic rationality in temporal argumentation.} \\
    In this paper, we have shown that abstract argumentation approaches are typically economically not rational when considering the normal expansion of argumentation frameworks and that economic rationality -- \emph{i.e.}, reference independence -- can be achieved by avoiding the addition of cycles in normal expansions.
    However, it can be assumed that many scenarios require well-defined approaches to handling argumentation cycles in an economically rational (reference independent) manner.
    To devise such approaches, it should be possible to build on a fundament of works on the resolution of cycles in argumentation graphs (also called 'loop-busting')~\cite{10.1093/logcom/exu012}, as well as on methods that enforce that particular arguments are in an extension~\cite{10.5555/3007337.3007364}.
    \item \textbf{Economic rationality and advanced argumentation frameworks.} \\
    In addition to preference-based and value-based argumentation, a range of other works extend Dung's notion of an argumentation framework, for example by assigning weights or intervals to attacks (\emph{e.g.}, probabilistic~\cite{li2011probabilistic} and possibilistic~\cite{nieves2011possibilistic} argumentation).
    Given that the rational man's expansion as established in this paper does not guarantee reference independence in the case of preference-/value-based argumentation, the exploration of the intersection of these approaches and economic rationality can be considered promising future research.
    \item \textbf{Bounded Rationality and Abstract Argumentation.} \\ As mentioned in the introduction, from the perspective of formal argumentation, it can be interesting to relax the \emph{clear preferences} property of the economically rational decision-maker, for example to consider three-valued labeling and \textquote{undecided} status of arguments, or to support quantitative notions of uncertainty.
    Such work can potentially be evaluated from an economics perspective, as a systematic relaxation of rationality constraints that can create new formal models of bounded rationality. In this context, the recently introduced family of weak admissible set-based semantics can potentially be useful~\cite{BaumannBU20}. (However, let us informally observe that the semantics introduced in~\cite{BaumannBU20} do not satisfy weak reference independence by considering the example argumentation frameworks $AF = (\{a\}, \{\}), AF' = (\{a, b, c\}, \{(a, b), (b, c), (c, a)\}), \\ AF \preceq_N AF'$.)
\end{itemize}

\appendix
\section*{Appendices}

\subsection*{Appendix A. Strong Monotony Properties and Reference Independence}
\label{app-monotony}
This appendix provides the following proofs, given maximal conflict-free semantics and normal expansions:
\begin{itemize}
    \item Strong monotony implies strong reference independence, but not vice versa.
    \item Strong rational monotony does not imply strong reference independence and vice versa.
    \item Strong cautious monotony does not imply strong reference independence and vice versa.
\end{itemize}
Let us first prove that strong reference independence does not imply strong monotony.
\begin{proposition}[Strong Reference Independence does not Imply Strong Monotony]
    There exist an argumentation semantics $\sigma$ whose extensions are maximal conflict-free sets (w.r.t. set inclusion) and argumentation frameworks $AF = (AR, Attacks), AF' = (AR', Attacks')$, such that $AF \preceq_N AF'$, such that the following statement does not hold true:
    \begin{align*}
        {}& \text{If } \forall E \in \sigma(AF), \forall E' \in \sigma(AF'), (E' \not \subseteq AR \lor E' = E) \\
        {}& \phantom{eee} \text{ then } \\
        {}& (\forall E \in \sigma(AF),  \forall E' \in \sigma(AF'), E \subseteq E')
    \end{align*}
\end{proposition}
\begin{proof}
    Let us provide a proof by counter-example.
    \begin{enumerate}
        \item Let $AF = (AR, Attacks) = (\{a\}, \{\})$, $AF' = (\{a, b\}, \{(b, a)\})$ and let us take CF2 semantics, denoted by $\sigma_{cf2}$, as an example of a semantics whose extensions are maximal conflict-free sets. Note that $AF \preceq_N AF'$.
        \item $\sigma_{cf2}(AF) = \{\{a\}\}$ and $\sigma_{cf2}(AF') = \{\{b\}\}$. Let $E = \{a\}, E \in \sigma_{cf2}(AF)$. $\exists E' \in \sigma(AF')$ such that $\{a\} \not \subseteq E'$. The condition for strong monotony does not hold.
        \item $\forall E' \in \sigma(AF')$ it holds true that $E' \not \subseteq AR$. The condition for strong reference independence holds. We have proven the proposition.
    \end{enumerate}
\end{proof}

Strong monotony implies strong reference independence.
\begin{proposition}[Strong Monotony Implies Strong Reference Independence]~\label{strong-monotony-implies-ref-dep}
    For all argumentation semantics $\sigma$ whose extensions are maximal conflict-free sets (w.r.t. set inclusion) and every two argumentation frameworks $AF = (AR, Attacks), AF' = (AR', \\ Attacks')$, such that $AF \preceq_N AF'$, the following statement holds true:
    \begin{align*}
       {}& \text{If }  (\forall E \in \sigma(AF), \forall E' \in \sigma(AF'), E \subseteq E') \\
       {}& \phantom{eee} \text{ then } \\
       {}& (\forall E \in \sigma(AF), \forall E' \in \sigma(AF'), E' \not \subseteq AR \lor E' = E)
    \end{align*}
\end{proposition}
\begin{proof}
    \phantom{eee}
    \begin{enumerate}
        \item By definition, $\forall E \in \sigma(AF), E' \in \sigma(AF')$, $E$ is a maximal conflict-free (w.r.t. set inclusion) subset of $AR$ and $E'$ is a maximal conflict-free (w.r.t. set inclusion) subset of $AR'$. Because $AF \preceq_N AF'$, it follows that, if $E \subseteq E'$, we have the following cases:
        \\ \emph{i)} E = E' or
        \\ \emph{ii)} $E \subset E'$, which implies that $\exists a \in E'$, such that $a \not \in E, a \in AR' \setminus AR$, from which it follows that $E' \not \subseteq AR$.
        \item Consequently, by i) and ii) the following statement holds true:
        \begin{align*}
            \forall E \in \sigma(AF), \forall E' \in \sigma(AF'), \text{ if } E \subseteq E' \text{ then } E' \not \subseteq AR \lor E' = E
        \end{align*}
        Hence, the proposition holds true.
    \end{enumerate}
\end{proof}
However, we can prove that strong \emph{cautious} monotony does not imply strong reference independence.
\begin{proposition}[Strong Cautious Monotony does not Imply Strong Reference Independence]
        There exist an argumentation semantics $\sigma$ whose extensions are maximal conflict-free sets (w.r.t. set inclusion) and argumentation frameworks $AF = (AR, Attacks), AF' = (AR', Attacks')$, such that $AF \preceq_N AF'$, such that the following statement does not hold true:
    \begin{align*}
        &{} \text{If } \forall E \in \sigma(AF), (\{(a,b) \mid (a,b) \in Attacks', a \in AR' \setminus AR, b \in E \}=\emptyset \text{ implies } \\
        &{} \quad \forall E' \in \sigma(AF'), E \subseteq E') \\
        &{} \quad \quad \text{ then } \\
        &{} (\forall E \in \sigma(AF), \forall E' \in \sigma(AF'),  E' \not \subseteq AR \lor E' = E)
    \end{align*}
\end{proposition}
\begin{proof}
    Let us provide a proof by counter-example.
    \begin{enumerate}
        \item Let $AF = (AR, Attacks) = (\{c, d\}, \{(c, d)\})$, $AF' = (\{c, d, e\}, \{(c, d), (d, e), (e, e), (e, c)\})$ and let us take stage semantics, denoted by $\sigma_{stage}$, as an example of a semantics whose extensions are maximal conflict-free sets. Note that $AF \preceq_N AF'$.
        \item $\sigma_{stage}(AF) = \{\{c\}\}$ and  $\sigma_{stage}(AF') = \{\{c\}, \{d\}\}$.
        \item Given the only extension $E = \{c\}, E \in \sigma_{stage}(AF)$, $\{(a,b) \mid (a,b) \in Attacks', a \in AR' \setminus AR, b \in E \}=\emptyset$ does not hold true and hence, $\forall E' \in \sigma_{stage}(AF'), \{(a,b) \mid (a,b) \in Attacks', a \in AR' \setminus AR, b \in E \}=\emptyset \text{ implies } E \subseteq E'$ holds true.
        \item However, it does not hold true that $\forall E \in \sigma(AF), \forall E' \in \sigma(AF'), (E' \not \subseteq AR \lor E' = E)$. We have proven the proposition.
    \end{enumerate}
\end{proof}
In a similar manner, we can prove that strong reference independence does not imply strong cautious monotony.
\begin{proposition}[Strong Reference Independence does not Imply Strong Cautious Monotony]
    There exist an argumentation semantics $\sigma$ whose extensions are maximal conflict-free sets (w.r.t. set inclusion) and argumentation frameworks $AF = (AR, Attacks), AF' = (AR', Attacks')$, such that $AF \preceq_N AF'$, such that the following statement does not hold true:
    \begin{align*}
         &{} \text{If } (\forall E \in \sigma(AF), \forall E' \in \sigma(AF'), E' \not \subseteq AR \lor E' = E) \\
         &{} \quad \quad \text{ then } \\
         &{} \forall E \in \sigma(AF), (\{(a,b) \mid (a,b) \in Attacks', a \in AR' \setminus AR, b \in E \}=\emptyset \text{ implies } \\
         &{} \quad \forall E' \in \sigma(AF'), E \subseteq E')
    \end{align*}
\end{proposition}
\begin{proof}
    We provide the following proof by counter-example.
    \begin{enumerate}
        \item Let $AF = (AR, Attacks) = (\{c, d, e\}, \{(c, d), (d, e), (e, c)\})$, \\ $AF' = (\{c, d, e, f\}, \{(c, d), (d, e), (e, c), (f, d)\})$ and let us take stage semantics, denoted by $\sigma_{stage}$, as an example of a semantics whose extensions are conflict-free sets. Note that $AF \preceq_N AF'$.
        \item $\sigma_{stage}(AF) = \{\{c\}, \{d\}, \{e\}\}$ and  $\sigma_{stage}(AF') = \{\{e, f\}\}$.
        \item Given the extension $E \in \sigma_{stage}(AF), E = \{a\}$, $\{(a,b) \mid (a,b) \in Attacks', a \in AR' \setminus AR, b \in E \}=\emptyset$ holds true.
        \item Given $E' = \{e, f\}$ as the only extension in $\sigma(AF')$ it holds true that $E' \not \subseteq AR$; the result does not violate the strong reference independence property.
        \item However, given $E' = \{e, f\}$ as the only extension in $\sigma(AF')$ and given $E = \{c\}, E \in \sigma_{stage}(AF)$, we have $E \not \subseteq E'$, which violates strong cautious monotony. We have proven the proposition.
    \end{enumerate}
\end{proof}
Let us now prove that strong rational monotony does not imply strong reference independence.
\begin{proposition}[Strong Rational Monotony does not Imply Strong Reference Independence]
    There exist an argumentation semantics $\sigma$ whose extensions are maximal conflict-free sets (w.r.t. set inclusion) and argumentation frameworks $AF = (AR, Attacks), AF' = (AR', Attacks')$, such that $AF \preceq_N AF'$, such that the following statement does not hold true:
    \begin{align*}
             &{} \quad \text{If } \forall E \in \sigma(AF), (\{(a,b) \mid (a,b) \in Attacks', a \in UE'_{new}, b \in E \}=\emptyset \text{ implies } \\
             &{} \quad \forall E' \in \sigma(AF'), E \subseteq E') \\ 
             &{} \quad \quad \text{ then } \\
             &{} \quad (\forall E \in \sigma(AF), \forall E' \in \sigma(AF'), (E' \not \subseteq AR \lor E' = E))),
        \end{align*}
        where $UE'_{new} = \bigcup_{E' \in \sigma(AF')} (E' \setminus AR)$.
\end{proposition}
\begin{proof}
    Let us provide a proof by counter-example.
        \begin{enumerate}
        \item Let $AF = (AR, Attacks) = (\{c, d\}, \{(c, d), (d, c)\})$ and let $AF' = (AR', Attacks') = (\{c, d, e\}, \{(c, d), (d, c), (d, e), (e, d), (e, c)\})$. Let us take stage semantics, denoted by $\sigma_{stage}$, as an example of a semantics whose extensions are maximal conflict-free sets. Note that $AF \preceq_N AF'$.
        \item $\sigma_{stage}(AF) = \{\{c\}, \{d\}\}$ and $\sigma_{stage}(AF') = \{\{d\}, \{e\}\}$.
        \item Given $\forall E \in \sigma(AF)$, $\{(a,b) \mid (a,b) \in Attacks', a \in UE'_{new}, b \in E \}=\emptyset$ does not hold true. Hence, $\{(a,b) \mid (a,b) \in Attacks', a \in UE'_{new}, b \in E \}=\emptyset \text{ implies } E \subseteq E'$ holds true.
        However, $\exists E \in \{\{c\}, \{d\}\}$ such that it does not hold true that $(\forall E' \in \sigma(AF'), (E' \not \subseteq AR \lor E' = E))$. We have proven the proposition.
    \end{enumerate}
\end{proof}
Finally, let us prove that strong reference independence does not imply strong rational monotony.
\begin{proposition}[Strong Reference Independence does not Imply Strong Rational Monotony]
    There exist an argumentation semantics $\sigma$ whose extensions are maximal conflict-free sets (w.r.t. set inclusion) and argumentation frameworks $AF = (AR, Attacks), AF' = (AR', Attacks')$, such that $AF \preceq_N AF'$, such that the following statement does not hold true:
    \begin{align*}
             &{}\text{If } (\forall E \in \sigma(AF), \forall E' \in \sigma(AF'), E' \not \subseteq AR \lor E' = E)) \\
             &{} \quad \quad \text{then} \\
             &{}\forall E \in \sigma(AF), (\{(a,b) \mid (a,b) \in Attacks', a \in UE'_{new}, b \in E \}=\emptyset \\
             &{} \quad \text{ implies } \forall E' \in \sigma(AF'), E \subseteq E'),
    \end{align*}
    where $UE'_{new} = \bigcup_{E' \in \sigma(AF')} (E' \setminus AR)$.
\end{proposition}
\begin{proof}
    Let us provide a proof by counter-example.
    \begin{enumerate}
        \item Let $AF = (AR, Attacks) = (\{c, d, e\}, \{(c, d), (d, e), (e, c)\})$, \\ $AF' = (\{c, d, e, f\}, \{(c, d), (d, e), (e, c), (f, d)\})$ and let us take stage semantics, denoted by $\sigma_{stage}$, as an example of a semantics whose extensions are maximal conflict-free sets. Note that $AF \preceq_N AF'$.
        \item $\sigma_{stage}(AF) = \{\{c\}, \{d\}, \{e\}\}$ and  $\sigma_{stage}(AF') = \{\{e, f\}\}$.
        \item $\forall E' \in \sigma(AF')$, it holds true that $E' \not \subseteq AR$.
        \item However, given the extension $E \in \sigma_{stage}(AF)$, $E = \{c\}$ it holds true that $\{(a,b) \mid (a,b) \in Attacks', a \in UE'_{new}, b \in E \}=\emptyset$ and $\exists E' \in \sigma(AF')$, such that $E \not \subseteq E'$, which violates strong rational monotony.
    \end{enumerate}
\end{proof}
\subsection*{Appendix B. Monotony and Reference Independence In Maximal Admissible Set-Based Semantics}
Let us prove that weak monotony implies weak reference independence and strong monotony implies strong reference independence in the case of maximal admissible set-based semantics.
\begin{proposition}[Weak Monotony Implies Weak Reference Independence]\label{prop-set-based}
    For all argumentation semantics $\sigma$ whose extensions are maximal admissible sets (w.r.t. set inclusion) and every two argumentation frameworks $AF = (AR, Attacks), AF' = (AR', Attacks')$, such that $AF \preceq_N AF'$, the following statement holds true:
    \begin{align*}
       {}& \text{If }  (\forall E \in \sigma(AF) \text{ it holds true that } \exists E' \in \sigma(AF'), E \subseteq E') \\
       {}& \phantom{eee} \text{ then } \\
       {}& (\forall E \in \sigma(AF) \text{ it holds true that } \exists E' \in \sigma(AF'), E' \not \subseteq AR \lor E' = E)
    \end{align*}
\end{proposition}
\begin{proof}
    Let us provide a proof by contradiction.
\begin{enumerate}
    \item $\forall E \in \sigma(AF), E' \in \sigma(AF')$ if $E \subseteq E'$, then we have the following cases:
    \\ \emph{i)} E = E' or
    \\ \emph{ii)} $E \subset E'$.
    \item We need to contradict \emph{ii)}. From \emph{ii)} it follows that if $E' \subseteq AR$ then $\exists a \not \in E, a \in AR' \cap AR$, such that $E \cup \{a\}$ is admissible.
    \item It follows that $E$ is not a maximal admissible set (w.r.t. set inclusion) of $AF$. Contradiction.
\end{enumerate}
\end{proof}

The same proof applies to the proposition that strong monotony implies strong reference independence in the case of maximal admissible set-based semantics.
\begin{corollary}
    For all argumentation semantics $\sigma$ whose extensions are maximal admissible sets (w.r.t. set inclusion) and every two argumentation frameworks $AF = (AR, Attacks), AF' = (AR', Attacks')$, such that $AF \preceq_N AF'$, the following statement holds true:
    \begin{align*}
       {}& \text{If }  (\forall E \in \sigma(AF), \forall E' \in \sigma(AF'), E \subseteq E') \\
       {}& \phantom{eee} \text{ then } \\
       {}& (\forall E \in \sigma(AF), \forall E' \in \sigma(AF'), E' \not \subseteq AR \lor E' = E)
    \end{align*}
\end{corollary}
The proof follows from the proof of Proposition~\ref{prop-set-based}.

Let us now provide the proof that weak rational monotony implies weak reference independence for maximal admissible set-based semantics.
\begin{proposition}[Weak Rational Monotony Implies Weak Reference Independence]\label{rational-monotony-not-ref-dep-2}
    For every argumentation semantics $\sigma$ whose extensions are maximal admissible sets (w.r.t. set inclusion) and every two argumentation frameworks $AF = (AR, Attacks), AF' = (AR', Attacks')$, such that $AF \preceq_N AF'$, the following statement holds true:
    \begin{align*}
             &{} \forall E \in \sigma(AF), \\
             &{} \quad (\text{If } (\{(a,b) \mid (a,b) \in Attacks', a \in UE'_{new}, b \in E \}=\emptyset \\
              &{} \quad \quad \text{ implies } \exists E' \in \sigma(AF'), \text{ such that } E \subseteq E') \\ 
             &{} \quad \quad \text{ then } \\
             &{}\quad (\exists E' \in \sigma(AF'), \text{ such that } (E' \not \subseteq AR \lor E' = E)))
        \end{align*},
        where $UE'_{new} = \bigcup_{E' \in \sigma(AF')} (E' \setminus AR)$.
\end{proposition}
\begin{proof}
    Given any $E \in \sigma(AF)$, we have two cases for which weak rational monotony is satisfied.
    \begin{description}
        \item[Case 1:] $\{(a,b) \mid (a,b) \in Attacks', a \in UE'_{new}, b \in E \} \neq \emptyset$: \\
        If $\{(a,b) \mid (a,b) \in Attacks', a \in UE'_{new}, b \in E \} \neq \emptyset$, then $\exists E' \in \sigma(AF'), \text{ such that } (E' \not \subseteq AR)$. Consequently, weak reference independence is satisfied.
         \item[Case 2:] $\{(a,b) \mid (a,b) \in Attacks', a \in UE'_{new}, b \in E \} = \emptyset$ and $\exists E' \in \sigma(AF'), \text{ such that } E \subseteq E'$: \\
         Because it holds true that $\exists E' \in \sigma(AF'), \text{ such that } E \subseteq E'$ and because all $\sigma$-extensions are maximal admissible sets (w.r.t. set inclusion), from the proof of Proposition~\ref{prop-set-based} it follows that weak reference independence is satisfied. We have proven the proposition.
    \end{description}
\end{proof}

\subsection*{Appendix C. Uniqueness in Argumentation Semantics}
\label{app-uniqueness}
This appendix provides an overview of which argumentation semantics satisfies the \emph{uniqueness} principle. We can rely on an analysis provided by Baumann as the basis of our overview~\cite{baumann2017nature}.
Note that we base the overview on the assumption that any argumentation framework is \emph{finite}.
Consequently, if, given an argumentation semantics, at least one extension is guaranteed given any finite argumentation framework, unique skeptical extensions are guaranteed.

\begin{table}[ht]
    \centering
    \def\arraystretch{1.5}
     \setlength\tabcolsep{5pt}
     \caption{Overview: compliance with the \emph{uniqueness} principle.}
     \vspace{5pt}
    \begin{tabular}{ | c | c | c |}
      \hline
      \phantom{e} & \emph{Credulous ($\sigma(AF)$)} & \emph{Skeptical ($\sigma^{\cap}(AF)$)} \\
      \hline
      Complete &  No  & Yes  \\
      Grounded &  Yes & Yes  \\
      Preferred &  No & Yes \\
      Stable &  No & No  \\
      Ideal &  Yes  & Yes \\
      Semi-stable & No & Yes \\
      Eager &  Yes & Yes \\
      Stage &  No & Yes \\
      CF2 &  No & Yes \\
      Stage2 &  No & Yes \\
      \hline
    \end{tabular}
    \label{principle-overview-1}
\end{table}

\acks{
This work was partially supported by the Wallenberg AI, Autonomous Systems and Software Program (WASP) funded by the Knut and Alice Wallenberg Foundation.
We thank Dov Gabbay, as well as the anonymous reviewers, for valuable feedback and discussions.}
\vskip 0.2in

\bibliography{references}
\bibliographystyle{acm}
\end{document}